\documentclass{article}

\usepackage{graphicx}
\usepackage{subcaption}
\usepackage{booktabs}
\usepackage{bm}
\usepackage{multicol,multirow}
\usepackage{natbib}
\usepackage{comment}
\usepackage{booktabs}

\usepackage{enumerate}   
\usepackage{amsmath}
\usepackage{amsfonts}
\usepackage{tikz}
\usetikzlibrary{positioning,arrows}
\usepackage{wrapfig}

\DeclareMathOperator*{\argmin}{arg\,min}

\newcommand{\tabincell}[2]{\begin{tabular}{@{}#1@{}}#2\end{tabular}}

\usepackage{amsmath}
\usepackage{appendix}
\usepackage{amssymb}

\usepackage{amsthm}
\usepackage{hyperref}
\usepackage[capitalise]{cleveref}
\Crefname{assumption}{Assumption}{Assumptions}

\theoremstyle{plain}
\newtheorem{theorem}{Theorem}
\newtheorem{lemma}[theorem]{Lemma}
\newtheorem{corollary}[theorem]{Corollary}
\theoremstyle{definition}
\newtheorem{assumption}{Assumption}

\newtheorem{remark}{Remark}
\newtheorem{example}{Example}

\newcommand{\ch}{\mathcal{H}}

\newcommand{\E}{\mathbb{E}}

\newcommand{\conv}{\mathrm{conv}}
\newcommand{\rE}{\mathrm{E}}
\newcommand{\rvar}{\mathrm{var}}
\newcommand{\rD}{\mathrm{d}}
\newcommand{\rd}{\mathrm{d}}
\newcommand{\bbN}{\mathbb{N}}

\newcommand{\epol}{\pi_e}
\newcommand{\bpol}{\pi_b}
\newcommand{\pre}{d_0}

\newcommand{\vect}[1]{\ensuremath{\mathbf{#1}}}

\newcommand{\Rfrak}{\mathfrak{R}}
\newcommand{\Fcal}{\mathcal{F}}
\newcommand{\Scal}{\mathcal{S}}
\newcommand{\Acal}{\mathcal{A}}
\newcommand{\Xcal}{\mathcal{X}}
\newcommand{\Wcal}{\mathcal{W}}
\newcommand{\RR}{\mathbb{R}}
\newcommand{\Qcal}{\mathcal{Q}}
\newcommand{\Vcal}{\mathcal{V}}
\newcommand{\Gcal}{\mathcal{G}}
\newcommand{\Rcal}{\mathcal{R}}

\newcommand{\bzero}{\mathbf{0}}

\newcommand{\ts}{\textstyle}
\newcommand{\Lw}{L_{\mathrm{w}}}
\newcommand{\Lwn}{L_{\mathrm{w},n}}
\newcommand{\Lq}{L_{\mathrm{q}}}
\newcommand{\Lqn}{L_{\mathrm{q},n}}
\newcommand{\hatw}{\hat{w}}
\newcommand{\hatwn}{\hat{w}_n}
\newcommand{\hatq}{\hat{q}}
\newcommand{\hatqn}{\hat{q}_n}
\newcommand{\Rw}{R_{\mathrm{w}}}
\newcommand{\Rwn}{R_{\mathrm{w},n}}
\newcommand{\Rq}{R_{\mathrm{q}}}

\newcommand{\RN}[1]{%
  \textup{\uppercase\expandafter{\romannumeral#1}}%
}

\usepackage{color,graphicx,lscape,longtable}

\def\boxit#1{\vbox{\hrule\hbox{\vrule\kern6pt\vbox{\kern6pt#1\kern6pt}\kern6pt\vrule}\hrule}}

\newcommand{\para}[1]{\noindent \textbf{#1}~~}

\usepackage{xcolor}
\newcount\Comments  
\newcommand{\kibitz}[2]{\ifnum\Comments=1\textcolor{#1}{#2}\fi}
\newcommand{\nj}[1]{\kibitz{blue}{[NJ: #1]}}

\usepackage{microtype}
\usepackage{graphicx}
\usepackage{booktabs} 

\usepackage{hyperref}

\usepackage[accepted]{icml2020}

\icmltitlerunning{Minimax Weight and Q-Function Learning for Off-Policy Evaluation}

\begin{document}

\twocolumn[
\icmltitle{Minimax Weight and Q-Function Learning for Off-Policy Evaluation}



\icmlsetsymbol{equal}{*}

\begin{icmlauthorlist}
\icmlauthor{Masatoshi Uehara}{ha}
\icmlauthor{Jiawei Huang}{uiuc}
\icmlauthor{Nan Jiang}{uiuc}
\end{icmlauthorlist}

\icmlaffiliation{ha}{Harvard University, Massachusetts ,  Boston, USA}
\icmlaffiliation{uiuc}{University of Illinois at Urbana-Champaign, Champaign, Illinois, USA}

\icmlcorrespondingauthor{Masatoshi Uehara}{ueharamasatoshi136@gmail.com}

\icmlkeywords{reinforcement learning, off-policy evaluation}

\vskip 0.3in
]

\printAffiliationsAndNotice{}  

\begin{abstract}
We provide theoretical investigations into off-policy evaluation in reinforcement learning using function approximators for (marginalized) importance weights and value functions. Our contributions include:\\
(1) A new estimator, MWL, that directly estimates importance ratios over the state-action distributions, removing the reliance on knowledge of the behavior policy as in prior work  \citep{liu2018breaking}. \\
(2) Another new estimator, MQL, obtained by swapping the roles of importance weights and value-functions in MWL. MQL has an intuitive interpretation of minimizing \emph{average Bellman errors} and can be combined with MWL in a doubly robust manner. \\
(3) Several additional results that offer further insights, including the sample complexities of MWL and MQL, their asymptotic optimality in the tabular setting, how the learned importance weights depend the choice of the discriminator class, and how our methods provide a unified view of some old and new algorithms in RL.
\end{abstract}

\section{Introduction}
In reinforcement learning (RL), off-policy evaluation (OPE) refers to the problem of estimating the performance of a new policy using historical data collected from a different policy, which is of crucial importance to the real-world applications of RL. The problem is genuinely hard as any unbiased estimator has to suffer a variance exponential in horizon in the worst case \citep{Li2015, jiang2016doubly}, known as the \emph{curse of horizon}. 

Recently, a new family of estimators based on marginalized importance sampling (MIS) receive significant attention from the community \citep{liu2018breaking, XieTengyang2019OOEf}, as they overcome the curse of horizon with relatively mild representation assumptions. The basic idea is to learn the marginalized importance weight that converts the state distribution in the data to that induced by the target policy, which sometimes has much smaller variance than the importance weight on action sequences used by standard sequential IS. Among these works, \citet{liu2018breaking} 
learn the importance weights by solving a minimax optimization problem defined with the help of a discriminator value-function class. 

In this work, we investigate more deeply the space of algorithms that utilize a value-function class and an importance weight class for OPE. Our main contributions are:
\begin{enumerate}[\textbullet]
\item (Section~\ref{sec:mwl}) A new estimator, MWL, that directly estimates importance ratios over the state-action distributions, removing the reliance on knowledge of the behavior policy as in prior work  \citep{liu2018breaking}.
\item (Section~\ref{sec:mql}) By swapping the roles of importance weights and Q-functions in MWL, we obtain a new estimator that learns a Q-function using importance weights as discriminators. 
The procedure and the guarantees of MQL exhibit an interesting symmetry w.r.t.~MWL. We also combine MWL and MQL in a doubly robust manner and provide their sample complexity guarantees (Section~\ref{sec:dr}).
\item (Section~\ref{sec:optimal_main}) We examine the statistical efficiency of MWL and MQL, and show that by modeling state-action functions, MWL and MQL are able to achieve the semiparametric lower bound of OPE in the tabular setting while their state-function variants fail to do so. 
\item Our work provides a unified view of many old and new algorithms in RL. For example, when both importance weights and value functions are modeled using the same linear class, we recover LSTDQ \citep{LagoudakisMichail2004LPI} and off-policy LSTD \citep{BertsekasDimitriP2009Pemf, dann2014policy} as special cases of MWL/MQL and their state-function variants. This gives LSTD algorithms a novel interpretation that is very different from the standard TD intuition. As another example,  (tabular) model-based OPE and step-wise importance sampling---two algorithms that are so different that we seldom connect them to each other---are both special cases of MWL. 
\end{enumerate}

\section{Preliminaries} \label{sec:prelim}
An infinite-horizon discounted MDP is often specified by a tuple $(\Scal, \Acal, P, \Rcal, \gamma)$ where $\Scal$ is the state space, $\Acal$ is the action space, $P: \Scal\times\Acal\to\Delta(\Scal)$ is the transition function, $\Rcal: \Scal\times\Acal\to\Delta([0, R_{\max}])$ is the reward function, and $\gamma \in [0, 1)$ is the discount factor. We also use $\Xcal:=\Scal \times\Acal$ to denote the space of state-action pairs. Given an MDP, a (stochastic) policy $\pi: \Scal\to \Delta(\Acal)$ and a starting state distribution $\pre \in \Delta(\Scal)$ together determine a distribution over trajectories of the form 
$s_0, a_0, r_0, s_1, a_1, r_1, \ldots$, where $s_0 \sim \pre$, $a_t \sim \pi(s_t)$, $r_t \sim \Rcal(s_t, a_t)$, and $s_{t+1} \sim P(s_t, a_t)$ for $t \ge 0$. 
The ultimate measure of the a policy's performance is the (normalized) expected discounted return: 
\begin{align}\label{eq:return}
\textstyle R_{\pi}:=(1-\gamma)\rE_{\pre,\pi}\left[\sum_{t=0}^{\infty}\gamma^{t}r_t\right],
\end{align}
where the expectation is taken over the randomness of the trajectory determined by the initial distribution and the policy on the subscript, and $(1-\gamma)$ is the normalization factor. 

A concept central to this paper is the notion of (normalized) discounted occupancy: 
\begin{align*}\textstyle
d_{\pi, \gamma} := (1-\gamma) \sum_{t=0}^\infty \gamma^{t} d_{\pi, t},
\end{align*}
where $d_{\pi, t} \in \Delta(\Xcal)$ is the distribution of $(s_t, a_t)$ under policy $\pi$. (The dependence on $\pre$ is made implicit.) We will sometimes also write $s \sim d_{\pi, \gamma}$ for sampling from its marginal distribution over states. An important property of discounted occupancy, which we will make heavy use of, is
\begin{align} \label{eq:Rpi}
R_\pi = \rE_{(s,a) \sim d_{\pi, \gamma},\, r \sim \Rcal(s,a)}[r].
\end{align}
It will be useful to define the policy-specific Q-function: 
$$\textstyle
Q^\pi(s,a) := \rE[\sum_{t=0}^\infty \gamma^t r_t | s_0 = s, a_0 = a;\, a_t \sim \pi(s_t) ~ \forall t>0].
$$
The corresponding state-value function is $V^\pi(s) := Q^\pi(s, \pi)$, where for any function $f$, $f(s, \pi)$ is the shorthand for $\rE_{a\sim \pi(s)}[f(s,a)]$. 
\paragraph{Off-Policy Evaluation (OPE)}~
We are concerned with estimating the expected discounted return of an \emph{evaluation policy} $\epol$ under a given initial distribution $\pre$, using data collected from a different \emph{behavior policy} $\bpol$. For our methods, we will consider the following data generation protocol, where we have a dataset consisting of $n$ i.i.d.~tuples $(s,a,r,s')$ generated according to the distribution:
$$
s \sim d_{\bpol}, a \sim \bpol(s), r \sim \Rcal(s,a), s' \sim P(s,a). 
$$
Here $d_{\bpol}$ is some exploratory state distribution that well covers the state space,\footnote{Unlike \citet{liu2018breaking},  we do not need to assume that $d_{\bpol}$ is $\bpol$'s discounted occupancy; see Footnote~\ref{ft:dpib} in Appendix~\ref{sec:mswl} for the reason, where we also simplify \citet{liu2018breaking}'s loss so that it does not rely on this assumption.} and the technical assumptions required on this distribution will be discussed in later sections. With a slight abuse of notation we will also refer to the joint distribution over $(s,a,r,s')$ or its marginal on $(s,a)$ as $d_{\bpol}$, e.g., whenever we write $(s,a,r,s') \sim d_{\bpol}$ or $(s,a) \sim d_{\bpol}$, the variables are always distributed according to the above generative process. We will use $\rE[\cdot]$ to denote the exact expectation, and use $\rE_n[\cdot]$ as its empirical approximation using the $n$ data points.

\paragraph{On the i.i.d.~assumption} Although we assume i.i.d.~data for concreteness and the ease of exposition, the actual requirement on the data is much milder: 
our method works as long as the empirical expectation (over $n$ data points) concentrates around the exact expectation w.r.t.~$(s,a,r,s') \sim d_{\bpol}$ for \emph{some} $d_{\bpol}$.\footnote{We assume $a\sim \pi_b(s)$ throughout the paper since this is required by previous methods which we would like to compare to. However, most of our derivations do not require that the data is generated from a single behavior policy (which is a common characteristic of behavior-agnostic OPE methods).} This holds, for example, when the Markov chain induced by $\bpol$ is ergodic, and our data is a single long trajectory generated by $\bpol$ without resetting. As long as the induced chain mixes nicely, it is well known that the empirical expectation over the single trajectory will concentrate, and in this case $d_{\bpol}(s)$ corresponds to the stationary distribution of the Markov chain.\footnote{We consider precisely this setting in Appendix~\ref{app:noniid} to solidify the claim that we do not really need i.i.d.ness.}

\section{Overview of OPE Methods}
\label{sec:ope}

\begin{table*}[!]
\centering
\hspace*{-1em}\begin{tabular}{c|c|c|c|c}
& $\bpol$ known? & Target object & Func.~approx. &  \tabincell{c}{Tabular \\ optimality}\\ \hline
MSWL \citep{liu2018breaking} & Yes & $w_{\epol/\bpol}^{\Scal}$ (Eq.\eqref{eq:mis_state}) & $w_{\epol/\bpol}^{\Scal} \in \mathcal{W}^{\Scal}$, {\color{blue} $V^{\epol} \in \mathcal{F}^{\Scal}$} (*)  & No \\ \hline
MWL (Sec~\ref{sec:mwl}) & No & $w_{\epol/\bpol}$   (Eq.\eqref{eq:mis_sa}) & $w_{\epol/\bpol} \in \mathcal{W}$, {\color{blue} $Q^{\epol} \in \conv(\mathcal{F})$}  & Yes \\ \hline
MQL (Sec~\ref{sec:mql}) & No &  $Q^{\epol}$ & $Q^{\epol} \in \mathcal{Q}$,~~  {\color{blue}$w_{\epol/\bpol} \in \conv(\mathcal{G})$}  & Yes \\ \hline
Fitted-Q & No & $Q^{\epol}$ & $\mathcal{Q}$ closed under $B^{\epol}$  & Yes 
\end{tabular}

\caption{Summary of some of the OPE Methods. For methods that require knowledge of $\bpol$, the policy can be estimated from data to form a ``plug-in'' estimator \citep[e.g.,][]{HannaJosiahP.2018ISPE}. 
In the function approximation column, we use blue color to mark the conditions for the discriminator classes for minimax-style methods. For \citet{liu2018breaking}, we use $\Wcal^\Scal$ and $\Fcal^\Scal$ for the function classes to emphasize that their functions are over the state space (ours are over the state-action space). Although they assumed $V^{\epol} \in \Fcal^\Scal$ (*), this assumption can also be relaxed to $V^{\epol} \in \conv(\Fcal^\Scal)$ as in our analyses. Also note that the assumption for the main function classes ($\Wcal^\Scal$, $\Wcal$, and $\Qcal$) can be relaxed as discussed in Examples~\ref{exm:oracle_mwl} and \ref{exm:oracle_mql}, and we put realizability conditions here only for simplicity. \label{tab:main}}
\end{table*}


\paragraph{Direct Methods} A straightforward approach to OPE is to estimate an MDP model from data, and then compute the quantity of interest from the estimated model. An alternative but closely related approach is to fit $Q^{\epol}$ directly from data using standard approximate dynamic programming (ADP) techniques, e.g., the policy evaluation analog of Fitted Q-Iteration \citep{ernst2005tree, le2019batch}. While these methods overcome the curse of dimensionality and are agnostic to the knowledge of $\bpol$, they often require very strong representation assumptions to succeed: for example, in the case of fitting a Q-value function from data, not only one needs to assume  realizability, that the Q-function class (approximately) captures $Q^{\epol}$, but the class also needs to be closed under Bellman update $B^{\epol}$ \citep{antos2008learning}, otherwise ADP can diverge in discounted problems \citep{tsitsiklis1997analysis} or suffer exponential sample complexity in finite-horizon problems \citep[Theorem 45]{Dan2018}; we refer the readers to \citet{ChenJinglin2019ICiB} for further discussions on this condition. When the function approximator fails to satisfy these strong assumptions, the estimator can potentially incur a high bias. 

\paragraph{Importance Sampling (IS)} IS forms an unbiased estimate of the expected return by collecting full-trajectory behavioral data and reweighting each trajectory according to its likelihood under $\epol$ over $\bpol$ \citep{precup2000eligibility}. Such a ratio can be computed as the cumulative product of the importance weight over action ($\frac{\epol(a|s)}{\bpol(a|s)}$) for each time step, which is the cause of high variance in IS: even if $\epol$ and $\bpol$ only has constant divergence per step, the divergence will be amplified over the horizon, causing the cumulative importance weight to have exponential variance, thus the ``curse of horizon''. Although techniques that combine IS and direct methods can partially reduce the variance, the exponential variance of IS simply cannot be improved when the MDP has significant stochasticity \citep{jiang2016doubly}. 

\paragraph{Marginalized Importance Sampling (MIS)} MIS improves over IS by observing that, if $\bpol$ and $\epol$ induces marginal distributions over states that have substantial overlap---which is often the case in many practical scenarios---then reweighting the reward $r$ in each data point $(s,a,r,s')$ with the following ratio
\begin{align} \label{eq:mis_state}
w^{\Scal}_{\epol/\bpol}(s) \cdot \frac{\epol(a|s)}{\bpol(a|s)}, ~\text{where}~~ w^{\Scal}_{\epol/\bpol}(s):=\frac{d_{\epol, \gamma}(s)}{d_{\bpol}(s)}
\end{align}
can potentially have much lower variance than reweighting the entire trajectory \citep{liu2018breaking}. 
The difference between IS and MIS is essentially performing importance sampling using Eq.\eqref{eq:return} vs.~Eq.\eqref{eq:Rpi}. However, the weight $w^{\Scal}_{\epol/\bpol}$ is not directly available and has to be estimated from data. \citet{liu2018breaking} proposes an estimation procedure that requires two function approximators, one for modeling the weighting function $w^{\Scal}_{\epol/\bpol}(s)$, and the other for modeling $V^{\bpol}$ which is used as a discriminator class for distribution learning. Compared to the direct methods, MIS only requires standard realizability conditions for the two function classes, though it also needs the knowledge of $\bpol$. A related method for finite horizon problems has been developed by \citet{XieTengyang2019OOEf}. 


\section{Minimax Weight Learning (MWL)} \label{sec:mwl}

In this section we propose a simple extension to \citet{liu2018breaking} that is agnostic to the knowledge of $\bpol$. The estimator in the prior work uses a discriminator class that contains $V^{\epol}$ to learn the marginalized importance weight on state distributions (see Eq.\eqref{eq:mis_state}). We show that as long as the discriminator class is slightly more powerful---in particular, it is a Q-function class that realizes $Q^{\epol}$---then we are able to learn the importance weight over state-action pairs directly:
\begin{align} \label{eq:mis_sa}
w_{\epol/\bpol}(s,a) :=\frac{d_{\epol, \gamma}(s, a)}{d_{\bpol}(s, a)}.
\end{align}
We can use it to directly re-weight the rewards without having to know $\bpol$, as $R_{\epol} = \Rw[w_{\epol/\bpol}] := \rE_{\bpol}[w_{\epol/\bpol}(s,a) \cdot r].$ 
It will be also useful to define $\Rwn[w] := \rE_{n}[w(s,a) \cdot r]$ as the empirical approximation of $\Rw[\cdot]$ based on $n$ data points.

Before giving the estimator and its theoretical properties, we start with two assumptions that we will use throughout the paper, most notably that the state-action distribution in data well covers the discounted occupancy induced by $\bpol$.

\begin{assumption}
Assume $\mathcal{X}=\mathcal{S}\times\mathcal{A}$ is a compact space. Let $\nu$ be its Lebesgue measure. \footnote{When $\nu$ is the counting measure for finite $\Xcal$, all the results hold with minor modifications.}
\end{assumption}

\begin{assumption} \label{asm:Cw}
There exists $C_w < +\infty$ such that $w_{\epol/\bpol}(s,a) \le C_{w}$ $\forall (s,a) \in \Xcal$. 
\end{assumption}

In the rest of this section, we derive the new estimator and provide its theoretical guarantee. Our derivation (Eqs.\eqref{eq:mwl_drv}--\eqref{eq:mwl_loss}) provides the  high-level intuitions for the method while only invoking basic and familiar concepts in MDPs (essentially, just Bellman equations). The estimator of \citet{liu2018breaking} can be also derived in a similar manner. 

\paragraph{Derivation}
Recall that it suffices to learn $w: \Xcal\to\RR$ such that $\Rw[w] = R_{\epol}$. This is equivalent to
\begin{align} \label{eq:mwl_drv}
\E_{\bpol}[w(s,a) \cdot r] = (1-\gamma)\E_{s\sim \pre}[Q^{\epol}(s, \epol)].
\end{align}
By Bellman equation, we have $\E[r|s,a] = \E[Q^{\epol}(s,a) - \gamma Q^{\epol}(s', \epol) |s,a]$. We use the RHS to replace $r$ in Eq.\eqref{eq:mwl_drv}, 
\begin{align} \label{eq:mwl_drv2}
\E_{\bpol}[w(s,a) \cdot (Q^{\epol}(s,a) - \gamma Q^{\epol}(s', \epol))] \nonumber\\ =  (1-\gamma)\E_{s\sim \pre}[Q^{\epol}(s, \epol)].
\end{align}
To recap, it suffices to find any $w$ that satisfies the above equation. Since we do not know $Q^{\epol}$, we will use a function class $\Fcal$ that (hopefully) captures $Q^{\epol}$, and find $w$ that minimizes (the absolute value of) the following objective function that measures the violation of Eq.\eqref{eq:mwl_drv2} over all $f\in\Fcal$: 
\begin{align} \label{eq:mwl_loss}
\Lw(w, f) := \rE_{(s,a,r,s') \sim d_{\bpol}}[\{\gamma w(s,a)  \cdot f(s', \epol) \\-w(s,a)f(s,a)\}]+ 
(1-\gamma)\rE_{s \sim \pre}[f(s,\epol)].
\end{align}
The loss is always zero when $w = w_{\epol/\bpol}$, so adding functions to $\Fcal$ does not hurt its validity. Such a solution is also unique if we require $\Lw(w,f) = 0$ for a rich set of functions and $d_{\bpol}$ is supported on the entire $\Xcal$, formalized as the following lemma:\footnote{All proofs of this paper can be found in the appendices.} 
\begin{lemma}\label{lem:ratio-estimation-informal}
$\Lw(w_{\epol/\bpol},f)=0$ $\forall f \in L^2(\mathcal{X},\nu):=\{f: \int f(s,a)^2\mathrm{d}\mathrm{\nu}<\infty \}$. Moreover, under additional technical assumptions,\footnote{As we will see, the identifiability of $w_{\epol/\bpol}$ is not crucial to the OPE goal, so we defer the technical assumptions to a formal version of the lemma in Appendix~\ref{app:mwl}, 
where we also show that the same statement holds when $\Fcal$ is an ISPD kernel; see Theorem~\ref{thm:key2}.} $w_{\epol/\bpol}$ is the only function that satisfies this.
\end{lemma}
This motivates the following estimator, which uses two function classes: a class $\Wcal: \Xcal\to\RR$ to model the $w_{\epol/\bpol}$ function, and another class $\Fcal: \Xcal\to\RR$ to serve as the discriminators:
\begin{align}\label{eq:minimax}
\hatw(s,a)=\argmin_{w \in \mathcal{W}}\max_{f\in \mathcal{F}}\Lw(w,f)^{2}. 
\end{align}
Note that this is the ideal estimator that assumes exact expectations (or equivalently, infinite amount of data). In reality, we will only have access to a finite sample, and the real estimator replaces $\Lw(w, f)$ with its sample-based estimation, defined as
\begin{align} \label{eq:minimax_n}
\Lwn(w,f) & := \rE_{n}[\{\gamma w(s,a)  f(s',\epol)-w(s,a)f(s,a)\}] \nonumber \\
& \qquad +(1-\gamma)\rE_{\pre}[f(s,\epol)].
\end{align}
 So the sample-based estimator is
$\hatwn(s,a) := \argmin_{w \in \mathcal{W}}\max_{f\in \mathcal{F}}\Lwn(w,f)^{2}.$ 
We call this estimation procedure MWL (minimax weight learning), and provide its main theorem below. 


\begin{theorem}\label{thm:ipw}
	For any given $w: \Xcal \to \RR$, define $\Rw[w]=\rE_{d_{\bpol}}[w(s,a)\cdot r]$. If $Q^{\epol} \in \conv(\Fcal)$, where $\conv(\cdot)$ denotes the convex hull of a function class, 
	\begin{align*}
&	|R_{\epol}-\Rw[w]|  \leq \max_{f\in \mathcal{F}} |\Lw(w,f)|, \\
&	|R_{\epol}-\Rw[\hatw]|  \leq \min_{w\in\mathcal{W}} \max_{f\in \mathcal{F}} |\Lw(w,f)|. 
	\end{align*}
\end{theorem}
A few comments are in order:
\begin{enumerate}[1.]
\item To guarantee that the estimation is accurate, all we need is $Q^{\epol} \in \conv(\Fcal)$, and $\min_{w}\max_{f}|\Lw(w, f)|$ is small. While the latter can be guaranteed by realizability of $\Wcal$, i.e., $w_{\epol/\bpol} \in \Wcal$, we show in an example below that realizability is sufficient but not always necessary: in the extreme case where $\Fcal$ only contains $Q^{\epol}$, even a constant $w$ function can satisfy $\max_{f}|\Lw(w,f)| = 0$ and hence provide accurate OPE estimation.

\begin{example}[Realizability of $\Wcal$ can be relaxed] \label{exm:oracle_mwl}
When $\mathcal{F}= \{Q^{\epol}\}$, as long as $w_0 \in \Wcal$ where $w_0$ is a constant function that always evaluates to $R_{\epol}/R_{\bpol}$, we have $\Rw[\hatw] = R_{\epol}$. See Appendix~\ref{app:oracle_mwl} for a detailed proof.
\end{example}
\item For the condition that $Q^{\epol} \in \conv(\Fcal)$, we can further relax the convex hull to the linear span, though we will need to pay the $\ell_1$ norm of the combination coefficients in the later sample complexity analysis. It is also straightforward to incorporate approximation errors (see Remark~\ref{rem:approx} in Appendix~\ref{app:mwl}) and we do not further consider these relaxations for simplicity. 

\item Although Eq.\eqref{eq:minimax} uses $\Lw(w,f)^2$ in the objective function, the square is mostly for optimization convenience and is not vital in determining the statistical properties of the estimator. In later sample complexity analysis, it will be much more convenient to work with the equivalent objective function that uses $|\Lw(w,f)|$ instead. 
\item When the behavior policy $\bpol$ is known, we can incorporate this knowledge by setting $\Wcal=\{s \mapsto w(s)\frac{\epol(a|s)}{\bpol(a|s)}: w \in \Wcal^{\Scal} \}$, where $\Wcal^{\Scal}$ is some function class over the state space. The resulting estimator is still different from \cite{liu2018breaking} since our discriminator class is still over the state-action space. 
\end{enumerate}



\subsection{Case Studies}
The estimator in Eq.\eqref{eq:minimax_n} requires solving a minimax optimization problem, which can be computationally challenging. Following \citet{liu2018breaking} we show that the inner maximization has a closed form solution when we choose $\Fcal$ to correspond to a reproducing kernel Hilbert space (RKHS) $\ch_K$ be a RKHS associated with kernel $K(\cdot,\cdot)$. We include an informal statement below and defer the detailed expression to Appendix~\ref{app:mwl_rkhs} due to space limit.

\begin{lemma}[Informal] \label{lem:mwl_rkhs}
When $\mathcal{F}=\{f \in (\Xcal \to \RR): \langle f,f\rangle_{\ch_K} \leq 1\}$, the term $\max_{f\in \mathcal{F}}\Lw(w,f)^{2}$ has a closed form expression. 
\end{lemma} 

As a further special case when both $\Wcal$ and $\Fcal$ are linear classes under the same state-action features $\phi: \Xcal\to\RR^d$. The resulting algorithm has a close connection to LSTDQ \citep{LagoudakisMichail2004LPI}, which we will discuss more explicitly later in Section~\ref{sec:mql}, Example~\ref{exm:linear_mql}. 

\begin{example}\label{exm:linear_mwl}
Let $w(s,a;\alpha)=\phi(s,a)^\top \alpha$ 
where $\phi(s,a)\in \RR^d$ is some basis function and $\alpha$ is the parameters. If we use the same linear function space as $\Fcal$, i.e., $\Fcal = \{(s,a) \mapsto \phi(s,a)^\top \beta: \beta \in \RR^d\}$, then the estimation of $\alpha$ given by MWL is (assuming the matrix being inverted is full-rank) 
\begin{align}\label{eq:linear}
& \hat{\alpha}=\rE_{n}[-\gamma  \phi(s',\epol)\phi(s,a)^{\top} + \phi(s,a)\phi(s,a)^{\top}]^{-1} \nonumber\\
& \qquad (1-\gamma)\rE_{s \sim \pre}[\phi(s,\epol)].
\end{align}
The sample-based estimator for the OPE problem is therefore $\Rwn[\hatwn] = \rE_n[r\phi(s,a)^{\top}] \hat{\alpha}$; see Appendix~\ref{app:mwl_linear} for a full derivation.   
\end{example}
Just as our method corresponds to LSTDQ in the linear setting, it is worth pointing out that the method of \citet{liu2018breaking}---which we will call MSWL (minimax state weight learning) for distinction and easy reference---corresponds to off-policy LSTD \citep{BertsekasDimitriP2009Pemf, dann2014policy}; see Appendix~\ref{sec:mswl} for details.

\subsection{Connections to related work}
\citet{ChowYinlam2019DBEo} has recently proposed a version of MIS with a similar goal of being agnostic to the knowledge of $\bpol$. In fact, their estimator and ours have an interesting connection, as our Lemma~\ref{lem:ratio-estimation-informal} can be obtained by taking the functional derivative of their loss function; we refer interested readers to \cref{app:dualdice} for details. That said, there are also important differences between our methods. First, our loss function can be reduced to single-stage optimization when using an RKHS discriminator, just as in \citet{liu2018breaking}. In comparison, the estimator of \citet{ChowYinlam2019DBEo} cannot avoid two-stage optimization. Second, they do not directly estimate $w_{\epol/\bpol}(s,a)$, and instead estimate $\nu^{*}(s,a)$ such that $\nu^{*}(s,a)-\gamma \rE_{s'\sim P(s,a)\,a'\sim \epol(s')}[\nu^{*}(s',a')]=w_{\epol/\bpol}(s,a)$, which is more indirect. 

In the special case of $\gamma = 0$, i.e., when the problem is a contextual bandit, our method essentially becomes  kernel mean matching when using an RKHS discriminator \citep{gretton2012kernel}, so MWL can be viewed as a natural extension of kernel mean matching in MDPs.

\section{Minimax Q-Function Learning (MQL)} \label{sec:mql}
In Section~\ref{sec:mwl}, we show how to use value-function class as discriminators to learn the importance weight function. In this section, by swapping the roles of $w$ and $f$, we derive a new estimator that learns $Q^{\epol}$ from data using importance weights as discriminators. The resulting objective function has an intuitive interpretation of \emph{average Bellman errors}, which has many nice properties and interesting connections to prior works in other areas of RL.

\paragraph{Setup} We assume that we have a class of state-action importance weighting functions $\Gcal \subset (\Xcal\to\RR)$ and a class of state-action value functions $\Qcal \subset (\Xcal\to\RR)$. To avoid confusion we do not reuse the symbols $\Wcal$ and $\Fcal$ in Section~\ref{sec:mwl}, but when we apply both estimators on the same dataset (and possibly combine them via doubly robust), it can be reasonable to choose $\Qcal = \Fcal$ and $\Gcal = \Wcal$. For now it will be instructive to assume that $Q^{\epol}$ is captured by $\Qcal$ (we will relax this assumption later), and the goal is to find $q \in \Qcal$ such that 
\begin{align}\label{eq:Rq}
\Rq[q] := (1-\gamma)\rE_{s\sim d_0}[q(s,\epol)]
\end{align} 
(i.e., the estimation of $R_{\epol}$ as if $q$ were $Q^{\epol}$) is an accurate estimate of $R_{\epol}$.\footnote{Note that $\Rq[\cdot]$ only requires knowledge of $\pre$ and can be computed directly. This is different from the situation in MWL, where $\Rw[\cdot]$ still requires knowledge of $d_{\bpol}$ even if the importance weights are known, and the actual estimator needs to use the empirical approximation $\Rwn[\cdot]$.}

\paragraph{Loss Function} The loss function of MQL is
\begin{align*}
\Lq(q,g)=\rE_{d_{\bpol}}[g(s,a)(r+\gamma q(s', \epol)-q(s,a))].
\end{align*}
As we alluded to earlier, if $g$ is the importance weight that converts the data distribution (over $(s,a)$) $d_{\bpol}$ to some other distribution $\mu$, then the loss becomes
$
\rE_{\mu}[r+ \gamma q(s', \epol) - q(s, a)],
$ 
which is essentially the average Bellman error defined by \citet{jiang2017contextual}. An important property of this quantity is that, if $\mu = d_{\epol, \gamma}$, 
then by (a variant of) Lemma 1 of \citet{jiang2017contextual}, we immediately have $R_{\epol} - \Rq[q]=$
$$\ts
\rE_{d_{\epol, \gamma}}[r+ \gamma q(s', \epol) - q(s, a)] (= \Lq(q, w_{\epol/\bpol})).
$$
Similar to the situation of MWL, we can use a rich function class $\Gcal$ to model $w_{\epol/\bpol}$, and find $q$ that minimizes the RHS of the above equation for all $g\in\Gcal$, which gives rise to the following estimator:
\begin{align*}
\hatq=\argmin_{q\in \mathcal{Q}}\max_{g\in \mathcal{G}}\Lq(q,g)^{2}. 
\end{align*}
We call this method MQL (minimax Q-function learning). Similar to Section~\ref{sec:mwl}, we use $\hatqn$ to denote the estimator based on a finite sample of size $n$ (which replaces $\Lq(q,g)$ with its empirical approximation $\Lqn(q,g)$), and develop the formal results that parallel those in Section~\ref{sec:mwl} for MWL. All proofs and additional results can be found in Appendix~\ref{app:mql}.

\begin{lemma}\label{lem:mql}
$\Lq(Q^{\epol},g)=0$ for $\forall g \in L^2(\mathcal{X},\nu)$.  Moreover, if we further assume that $d_{\bpol}(s,a) > 0~\forall (s,a)$, then $Q^{\epol}$ is the only function that satisfies such a property.
\end{lemma}

Similar to the case of MWL, we show that under certain representation conditions, the estimator will provide accurate estimation to $R_{\epol}$.
\begin{theorem}\label{thm:minimaxq}
The following holds if $w_{\epol/\bpol} \in \conv(\mathcal{G})$:
\begin{align*}
&|R_{\epol}-\Rq[q] | \leq \max_{g \in \mathcal{G}}|\Lq(q,g)|, \\
& |R_{\epol}-\Rq[\hatq] |  \leq \min_{q \in \mathcal{Q}}\max_{g \in \mathcal{G}}|\Lq(q,g)|.  
\end{align*}
\end{theorem}

\subsection{Case Studies}
We proceed to give several special cases of this estimator corresponding to different choices of $\Gcal$ to illustrate its properties. In the first example, we show the analogy of Example~\ref{exm:oracle_mwl} for MWL, which demonstrates that requiring $\min_{q}\max_g \Lq(q,g) = 0$ is weaker than realizability $Q^{\epol} \in \Qcal$: 

\begin{example}[Realizability of $\Qcal$ can be relaxed] \label{exm:oracle_mql}
When $\mathcal{G}= \{w_{\epol/\bpol}\}$, as long as $q_0 \in \Qcal$, where $q_0$ is a constant function that always evaluates to $R_{\epol}/(1-\gamma)$, we have $\Rq[\hatq] = R_{\epol}$. See Appendix~\ref{app:oracle_mql} for details.
\end{example}

Next, we show a simple and intuitive  example where $ w_{\epol/\bpol} \notin \Gcal$ but $ w_{\epol/\bpol} \in \conv(\Gcal)$, i.e., there are cases where relaxing $\Gcal$ to its convex hull yields stronger representation power and the corresponding theoretical results provide a better description of the algorithm's behavior.

\begin{example}
Suppose $\Xcal$ is finite. Let $\Qcal$ be the tabular function class, and $\Gcal$ is the set of state-action indicator functions.\footnote{Strictly speaking we need to multiply these indicator functions by $C_w$ to guarantee $ w_{\epol/\bpol} \in \conv(\Gcal)$; see the comment on linear span after Theorem~\ref{thm:ipw}.} Then $w_{\epol/\bpol} \notin \Gcal$ but $w_{\epol/\bpol} \in \conv(\Gcal)$, and $\Rq[\hatq] = 0$. Furthermore, the sample-based estimator $\hatqn$ coincides with the model-based solution, as $\Lqn(q, g) = 0$ for each $g$ is essentially the Bellman equation on the corresponding state-action pair in the estimated MDP model. (In fact, the solution remains the same if we replace $\Gcal$ with the tabular function class.)
\end{example}

In the next example, we choose $\mathcal{G}$ to be a rich $L^{2}$-class with bounded norm, and recover the usual (squared) Bellman error as a special case. A similar example has been given by \citet{FengYihao2019AKLf}. \begin{example}[$L^{2}$-class]
When $\mathcal{G}=\{g: \rE_{d_{\bpol}}[g^{2}]\leq 1 \}$, 
\begin{align*}
\max_{g\in\mathcal{G}}\Lq(q,g)^{2}=\rE_{d_{\bpol}}[\left((B^{\epol}q)(s,a)-q(s,a)\right)^{2}], 
\end{align*}
where $B^\pi$ is the Bellman update operator $(B^\pi q)(s,a) := \rE_{r\sim \Rcal(s,a), s' \sim P(s,a)}[r + \gamma q(s',\pi)]$.
\end{example}

Note that the standard Bellman error cannot be directly estimated from data when the state space is large, even if the $\Qcal$ class is realizable \citep{SzepesvariCsaba2005Ftbf,sutton2018reinforcement,ChenJinglin2019ICiB}. From our perspective, this difficulty can be explained by the fact that squared Bellman error corresponds to an overly rich discriminator class that demands an unaffordable sample complexity.

The next example is RKHS class which yields a closed-form solution to the inner maximization as usual. 
\begin{example}[RKHS class]
When $\mathcal{G}=\{g(s,a); \langle g,g \rangle_{\ch_K}\leq 1 \}$, we have the following: 
\begin{lemma}\label{lem:kernel2}
Let $\mathcal{G}=\{g(s,a); \langle g,g \rangle_{\ch_K}\leq 1 \}$. Then, 
$\max_{g\in \mathcal{G}}\Lq(q,g)^{2} = \rE_{d_{\bpol}}[\Delta^{q}(q;s,a,r,s')\Delta^{q}(q;\tilde s,\tilde a,\tilde r,\tilde s')K((s,a),(\tilde s, \tilde a))],$ 
where $\Delta^{q}(q;s,a,r,s')=g(s,a)(r+\gamma q(s', \epol)-q(s,a)).$
\end{lemma}
\end{example}

Finally, the linear case.
\begin{example} \label{exm:linear_mql}
Let $q(s,a;\alpha)=\phi(s,a)^\top \alpha$ 
where $\phi(s,a)\in \RR^d$ is some basis function and $\alpha$ is the parameters. If we use the same linear function space as $\Gcal$, i.e., $\Gcal = \{(s,a) \mapsto \phi(s,a)^\top \beta: \beta\in\RR^d\}$, then MQL yields $\hat{\alpha}$:
\begin{align*}\ts
\rE_{n}[-\gamma \phi(s,a) \phi(s',\epol)^{\top} + \phi(s,a)\phi(s,a)^{\top}]^{-1} \rE_{n}[r\phi(s,a)].
\end{align*}
The derivation is similar to that of Example~\ref{exm:linear_mwl} (Appendix~\ref{app:mwl_linear}) and omitted. )The resulting $q(s,a;\hat \alpha)$ as an estimation of $Q^{\epol}$ is precisely LSTDQ \citep{LagoudakisMichail2004LPI}.
In addition, the final OPE estimator $\Rq[\hatqn] = (1-\gamma) \rE_{s\sim d_0}[q(s, \epol; \hat \alpha)]$ 
is the same as $\Rwn[\hatwn]$ when $\Wcal$ and $\Fcal$ are the same linear class (Example~\ref{exm:linear_mwl}). 
\end{example}

\subsection{Connection to Kernel Loss \citep{FengYihao2019AKLf}}\label{sec:mvl}

\citet{FengYihao2019AKLf} has recently proposed a method for value-based RL. By some transformations, we may rewrite their loss over state-value function $v: \Scal\to\RR$ as
\begin{align} \label{eq:onpolicy_mvl}
\max_{g \in \mathcal{G}^\Scal} \left(\rE_{\epol}[\{r + \gamma v(s')-v(s)\}g(s)]\right)^{2},
\end{align}
where $\Gcal^\Scal$ is an RKHS over the state space. 
While their method is very similar to MQL when written as the above expression, they focus on learning a state-value function and need to be on-policy for policy evaluation. In contrast, our goal is OPE (i.e., estimating the expected return instead of the value function), and we learn a Q-function as an intermediate object and hence are able to learn from off-policy data. More importantly, the importance weight interpretation of $g$ has eluded their paper and they interpret this loss purely from a kernel perspective. In contrast, by leveraging the importance weight interpretation, we are able to establish approximation error bounds based on representation assumptions that are fully expressed in quantities directly defined in the MDP. We also note that their loss for policy optimization can be similarly interpreted as minimizing average Bellman errors under a set of distributions.

Furthermore, it is easy to extend their estimator to the OPE task using knowledge of $\bpol$, which we call MVL; see Appendix~\ref{app:mvl} for details. 	
Again, just as we discussed in Appendix~\ref{sec:mswl} on MSWL, when we use linear classes for both value functions and importance weights, these two estimators become two variants of off-policy LSTD \citep{dann2014policy, BertsekasDimitriP2009Pemf} and coincide with MSWL and its variant.  

\section{Doubly Robust Extension and Sample Complexity of MWL \& MQL} \label{sec:dr}
In the previous sections we have seen two different ways of using a value-function class and an importance-weight class for OPE. 
Which one should we choose?

In this section we show that there is no need to make a choice. In fact, we can combine the two estimates naturally through the doubly robust trick \citep{KallusNathan2019EBtC} (see also \citep{tang2019harnessing}), whose population version is: 
\begin{align} \label{eq:dr}
&R[w,q]=(1-\gamma)\rE_{\pre}[q(s, \epol)] \nonumber\\
& \quad +\rE_{d_{\bpol}}[w(s,a)\{r+\gamma q(s', \epol)-q(s,a)\}]. 
\end{align}
As before, we write $R_n[w,q]$ as the empirical analogue of $R[w,q]$. While $w$ and $q$ are supposed to be the MWL and MQL estimators in practice, in this section we will sometimes treat $w$ and $q$ as arbitrary functions from the $\Wcal$ and $\Qcal$ classes to keep our results general. 
By combining the two estimators, we obtain the usual doubly robust property, that when either $w=w_{\epol/\bpol}$ or $q=Q^{\epol}$, we have $R[w,q] = R_{\epol}$, that is, 
as long as either one of the models works well, the final estimator behaves well.\footnote{See \citet[Theorem 11,12]{KallusNathan2019EBtC} for formal statements.} 

Besides being useful as an estimator, Eq.\eqref{eq:dr} also provides a unified framework to analyze the previous estimators, which are all its special cases: Note that $R[w, \bzero] = \Rw[w]$ and $R[\bzero, q] = \Rq[q]$, where $\bzero$ means a constant function that always evaluates to $0$.
Below we first prove a set of results that unify and generalize the results in Sections~\ref{sec:mwl} and \ref{sec:mql}, and then state the sample complexity guarantees for the proposed estimators.

\begin{lemma}\label{lem:db}
$R[w,q]-R_{\epol}=\rE_{d_{\bpol}}[\{w(s,a)-w_{\epol/\bpol}(s,a)\}\{\gamma V^{\epol}(s')-\gamma v(s')+q(s,a)-Q^{\epol}(s,a)\}].$
\end{lemma}

\begin{theorem}\label{thm:db2}
Fixing any $q'\in\mathcal{Q}$, if $[Q^{\epol}-q']\in \conv(\mathcal{F})$,
\begin{align*}
&|R[w,q']-R_{\epol}|\leq \max_{f \in \mathcal{F}}|\Lw(w,f)|, \\ 	
&|R[\hatw,q']-R_{\epol}|\leq \min_{w \in \mathcal{W}}\max_{f \in \mathcal{F}}|\Lw(w,f)|. 
\end{align*}
Similarly, fixing any $w'\in \mathcal{W}$, if $ [w_{\epol/\bpol}-w'] \in \conv(\mathcal{G})$, 
\begin{align*}
& |R[w',q]-R_{\epol}|\leq \max_{g \in \mathcal{G}}|\Lq(q,g)|, \\
&|R[w',\hatq]-R_{\epol}|\leq \min_{q \in  \mathcal{Q}}\max_{g \in \mathcal{G}}|\Lq(q,g)|. 
\end{align*}
\end{theorem}

\begin{remark}
	When $q'= \bzero$, the first statement is reduced to Theorem \ref{thm:ipw}. When $w' = \bzero$, the second statement is reduced to Theorem~\ref{thm:minimaxq}.
\end{remark}

\begin{theorem}[Double robust inequality for discriminators (i.i.d case)]\label{thm:radema}
	Recall that 
	\begin{align*}\ts
	\hatwn &= \argmin_{w \in \mathcal{W}}\max_{f \in \mathcal{F}}\Lwn(w,f)^2,\\
	\hatqn &=\argmin_{q \in  \mathcal{Q}}\max_{g \in \mathcal{G}}\Lqn(q,g)^2,
	\end{align*}
	where $\Lwn$ and $\Lqn$ are the empirical losses based on  a set of $n$ i.i.d samples. We have the following two statements. 
	
	(1) Assume $ [Q^{\epol}-q'] \in \conv(\mathcal{F})$ for some $q'$, and $\forall f\in \mathcal{F},\|f\|_\infty<C_{f}$.
	Then, with probability at least $1-\delta$, 
	\begin{align*}
	|R[\hatwn,q']-R_{\epol}| \lnsim \min_{w \in \mathcal{W}}\max_{f \in \mathcal{F}}|\Lw(w,f)| \\
	+ \mathfrak{R}_{n}(\mathcal{F},\mathcal{W})+C_f C_w\sqrt{\frac{\log(1/\delta)}{n}}
	\end{align*}
	where $\mathfrak{R}_{n}(\mathcal{W},\mathcal{F})$ is the Rademacher complexity \footnote{See \citet{bartlett2003rademacher} for the definition.} of the function class $\{(s,a,s') \mapsto w(s,a)(\gamma f(s',\epol)-f(s,a))~:~~ w\in \mathcal{W}, f\in \mathcal{F}\}.$
	 
	(2) Assume $ [w_{\epol/\bpol}-w']  \in \conv(\mathcal{G})$ for some $w'$, and  $\forall g\in \mathcal{G},\|g\|_\infty<C_{g}$. Then, with probability at least $1-\delta$, 
	\begin{align*}
	&|R[w', \hatqn]-R_{\epol}| \lnsim \min_{q \in \mathcal{Q}}\max_{g \in \mathcal{G}}|\Lq(q,g)| \\
	&+\mathfrak{R}_{n}(\mathcal{Q},\mathcal{G})+C_g\frac{R_{\mathrm{max}}}{(1-\gamma)}\sqrt{\frac{\log(1/\delta)}{n}}, 
	\end{align*}
	where 
	$\mathfrak{R}_{n}(\mathcal{Q},\mathcal{G})$ is the Rademacher complexity   of the function class 
	$\{(s,a, r,s') \mapsto g(s,a)\{r+\gamma q(s',\epol)-q(s,a)\}~:~~ q\in \mathcal{Q}, g\in \mathcal{G}\}.$ 
\end{theorem}
Here $A \lnsim B$ means inequality without an (absolute) constant. 
Note that we can immediately extract the sample complexity guarantees for the MWL and the MQL estimators as the corollaries of this general guarantee by letting $q' = \bzero$ and $w' = \bzero$.\footnote{Strictly speaking, when $q' = \bzero$, $R[\hatwn, q'] = \Rw[\hatwn]$ is very close to but slightly different from the sample-based MWL estimator $\Rwn[\hatwn]$, but their difference can be bounded by a uniform deviation bound over the $\Wcal$ class in a straightforward manner. The MQL analysis does not have this issue as $\Rq[\cdot]$ does not require empirical approximation.} In Appendix~\ref{app:noniid} we also extend the analysis to the non-i.i.d.~case and show that similar results can be established for $\beta$-mixing data.


\section{Statistical Efficiency in the Tabular Setting} \label{sec:optimal_main}
As we have discussed earlier, both MWL and MQL are equivalent to LSTDQ when we use the same linear class for all function approximators. Here we show that in the tabular setting, which is a special case of the linear setting, 
MWL and MQL can achieve the semiparametric lower bound of OPE \citep{NathanUehara2019}, because they coincide with the model-based solution. This is a desired property that many OPE estimators fail to obtain, including MSWL and MVL. 

\begin{theorem}\label{thm:optimal_main}
	Assume the whole data set $\{(s,a,r,s')\}$ is geometrically ergodic \footnote{Regarding the definition, refer to \cite{MeynS.P.SeanP.2009Mcas}}. Then, in the tabular setting, $\sqrt{n}(\Rwn[\hatwn]-R_{\epol})$ and $\sqrt{n}(\Rq[\hatqn]-R_{\epol})$ weakly converge to the normal distribution with mean $0$ and variance
	$$
	\rE_{d_{\bpol}}[w^2_{\epol/\bpol}(s,a)(r+\gamma V^{\epol}(s')-Q^{\epol}(s,a))^{2}].
	$$
This variance matches the semiparametric lower bound for OPE given by \citet[Theorem 5]{NathanUehara2019}.
\end{theorem}

The details of this theorem and further discussions can be found in Appendix~\ref{sec:efficiency}, where we also show that MSWL and MVL have an asymptotic variance greater than this lower bound. To back up this theoretical finding, we also conduct experiments in the Taxi environment \citep{DietterichT.G.2000HRLw} following \citet[Section  5]{liu2018breaking}, and show that MWL performs significantly better than MSWL in the tabular setting; see 
Appendix~\ref{sec:numerical} for details. It should be noted, however, that our optimal claim is asymptotic, whereas explicit importance weighting of MSWL and MVL may provide strong regularization effects and hence preferred in the regime of insufficient data; we leave the investigation to future work. 



\section{Experiments}
We empirically demonstrate the effectiveness of our methods and compare them to baseline algorithms  in CartPole with function approximation. 
We compare MWL \& MQL to MSWL \citep[with estimated behavior policy]{liu2018breaking} and DualDICE \citep{ChowYinlam2019DBEo}. We use neural networks with 2 hidden layers as function approximators for the main function classes for all methods, and use an RBF kernel for the discriminator classes (except for DualDICE); due to space limit we defer the detailed settings to Appendix~\ref{app:exp_fa}. Figure~\ref{fig:exp_fa} shows the log MSE of relative errors of different methods, where MQL appears to the best among all methods. 
Despite that these methods require different function approximation capabilities and it is difficult to compare them apple-to-apple, the results still show that MWL/MQL can achieve similar performance to related algorithms and sometimes outperform them significantly. 

\begin{figure}[ht]
	\centering
		\includegraphics[scale=0.26]{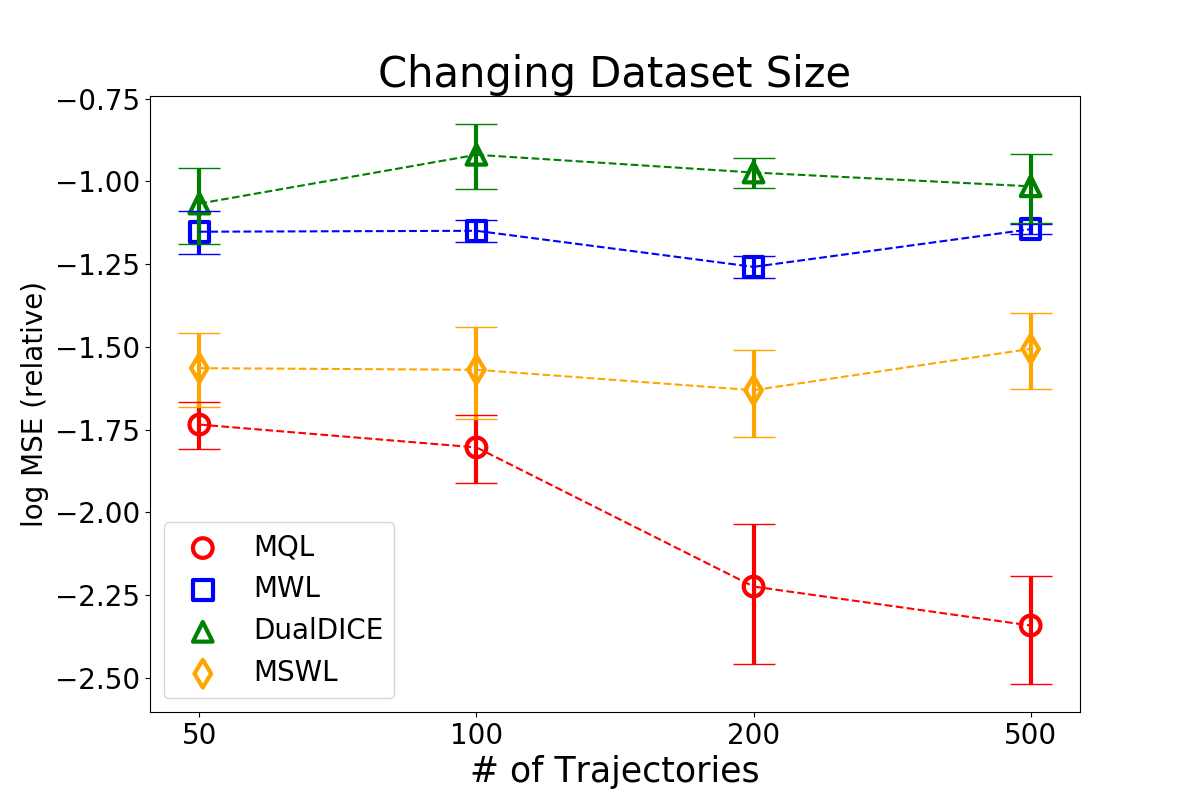}
	\caption{Accuracy of OPE methods as a function of sample size. Error bars show $95\%$ confidence intervals. \label{fig:exp_fa}}
\end{figure}

\section{Discussions} \label{sec:further}
We conclude the paper with further discussions. 

\para{On the dependence of $\hatw$ on $\Fcal$}
In Example~\ref{exm:oracle_mwl}, we have shown that with some special choice of the discriminator class, MWL can pick up very simple weighting functions---such as constant functions---that are very different from the ``true'' $w_{\epol/\bpol}$ and nevertheless produce accurate OPE estimates with very low variances. 
Therefore, the function $w$ that satisfies $\Lw(w,f)=0~ \forall f\in\Fcal$ may not be unique, and the set of feasible $w$  highly depends on the choice of $\Fcal$. This leads to several further questions, 
such as how to choose $\Fcal$ to allow for these simple solutions, and how regularization can yield simple functions for better bias-variance trade-off. \footnote{Such a trade-off is relatively well understood in contextual bandits \citep{KallusNathan2017GOMM,HirshbergDavidA.2017AMLE}, though extension to sequential decision-making is not obvious.} 
In case it is not clear  that the set of feasible $w$ generally depends on $\Fcal$, we provide two additional examples which are also of independent interest themselves. The first example shows that standard step-wise IS can be viewed as a special case of MWL, when we choose a very rich discriminator class of history-dependent functions. 

\begin{example}[Step-wise IS as a Special Case of MWL] \label{exm:is}
Every episodic MDP can be viewed as an equivalent MDP whose state is the \emph{history} of the original MDP. 
The marginal density ratio in this history-based MDP is essentially the cumulative product of importance weight used in step-wise IS, and from Lemma~\ref{lem:characterization} we know that such a function is the unique minimizer of MWL's population loss if $\Fcal$ is chosen to be a sufficiently rich class of functions over histories. See Appendix~\ref{app:stepIS} for more details on this example.
\end{example}

\begin{example}[Bisimulation]
Let $\phi$ be a bisimulation state abstraction (see \citet{li2006towards} for definition). If $\epol$ and $\bpol$ only depend on $s$ through $\phi(s)$, and $\Fcal$ only contains functions that are piece-wise constant under $\phi$, then $w(s,a) = \frac{d_{\epol,\gamma}(\phi(s),a)}{d_{\bpol}(\phi(s),a)}$ also satisfies $\Lw(w, f) = 0,~ \forall f\in\Fcal$. 
\end{example}

\para{Duality between MWL and MQL} From Sections~\ref{sec:mwl} and \ref{sec:mql}, one can observe an obvious symmetry between MWL and MQL from the estimation procedures to the guarantees, which reminds us a lot about the duality between value functions and distributions in linear programming for MDPs. Formalizing this intuition is an interesting direction.\footnote{See the parallel work by \citet{nachum2019algaedice} and the follow-up work of \citet{jiang2020minimax} for some intriguing discussions on this matter.} 


\section*{Acknowledgements}
We would like to thank the anonymous reviewers for
their insightful comments and suggestions. 

Masatoshi Uehara was supported in part by MASASON Foundation. 

\bibliographystyle{chicago}
\bibliography{ic,pfi3}

\appendix
\onecolumn
\newpage 
\begin{table}[!h]
	\centering
	\caption{Table of Notations
	\label{tab:my_label}} 
	\begin{tabular}{l|l}
		$\epol,\bpol$ & Evaluation policy, Behavior policy \\
		$\{(s_i,a_i,r_i,s'_i)\}_{i=1}^{n}$ & Finite sample of data  \\
		$(\mathcal{S},\mathcal{A},P, \mathcal{R},\gamma, \pre)$,\,$\mathcal{X}$ & MDP, $\mathcal{X}= \mathcal{S} \times\mathcal{A}$\\
		$d_{\bpol}$  &  Data distribution over $(s,a,r,s')$ or its marginals \\
		$d_{\epol,\gamma}$  & Normalized discounted occupancy induced by $\epol$ \\
		$(s,a) \sim d_0 \times \epol$ & $s \sim d_0$, $a \sim \epol(s)$\\
		$\beta_{\epol/\bpol}(a,s)$ & Importance weight on action:  $\epol(a|s)/\bpol(a|s)$\\
		$w_{\epol/\bpol}(s,a)$ & $d_{\epol, \gamma}(s,a)/d_{\bpol}(s,a)$   \\
		$C_w$ & Bound on $\|w_{\epol/\bpol}\|_\infty$ \\
		$V^{\epol}$ & Value function  \\
		$Q^{\epol}$ & Q-value function \\
		$R_{\epol}$ & Expected discounted return of $\epol$ \\
		$\Rw[\cdot]$ & OPE estimator using $(\cdot)$ as the weighting function (population version) \\
		$\Rwn[\cdot]$ & OPE estimator using $(\cdot)$ as the weighting function (sample-based version) \\
		$\Rq[\cdot]$ & OPE estimator using $(\cdot)$ as the approximate Q-function \\
		$\rE_{n}$ & Empirical approximation \\
		$\mathcal{W}$, $\mathcal{F}$ & Function classes for MWL \\
		$\mathcal{Q}$, $\mathcal{G}$ & Function classes for MQL \\
		$ \langle \cdot,\cdot\rangle_{\ch_K}$   &  Inner product of RKHS with a kernel $K$\\
		$\conv(\cdot)$ & convex hull \\
		$\nu$ & Uniform measure over the compact space $\mathcal{X}$ \\
		$L^{2}(\mathcal{X},\nu)$ & $L^2$-space on $X$ with respect to measure $\nu$ \\
		$\mathfrak{R}_n(\cdot) $ & Rademacher complexity  \\
		$\lnsim$  & Inequality without constant \\
	\end{tabular}
\end{table}

\section{Proofs and Additional Results of Section~\ref{sec:mwl} (MWL)}
\label{app:mwl}

We first give the formal version of Lemma~\ref{lem:ratio-estimation-informal}, which is Lemmas~\ref{lem:characterization} and \ref{lem:ratio-estimation} below, and then provide their proofs.

\begin{lemma}\label{lem:characterization}
For any function $g(s,a)$, define the map; $g\to \delta(g,s',a')$;
	\begin{align*}
	\delta(g,s',a')=\gamma \int P(s'|s,a)\epol(a'|s')g(s,a)\mathrm{d}\nu(s,a)-g(s',a')+(1-\gamma)\pre(s')\epol(a'|s'). 
	\end{align*}   
Then, $	\delta(d_{\epol, \gamma},s',a')=0\,\forall(s',a')$. 
\end{lemma}

\begin{proof}
We have 
\begin{align*}
    d_{\epol,\gamma}(s',a')&=(1-\gamma)\sum_{t=0}^{\infty}\gamma^t d_{\epol,t}(s',a') \\
    &=(1-\gamma)\left\{\pre(s')\epol(a'|s')+ \sum_{t=1}^{\infty}\gamma^t d_{\epol,t}(s',a')\right\}\\
    &=(1-\gamma)\left\{\pre(s')\epol(a'|s')+  \sum_{t=0}^{\infty}\gamma^{t+1} d_{\epol,t+1}(s',a')\right\}\\
    &=(1-\gamma)\left\{\pre(s')\epol(a'|s')+ \gamma \sum_{t=0}^{\infty} \int P(s'|s,a)\epol(a'|s')\gamma^t d_{\epol,t}(s,a)\mathrm{d}\nu(s,a)\right\}\\
    &=(1-\gamma)\pre(s')\epol(a'|s')+ \gamma \int  P(s'|s,a)\epol(a'|s') d_{\epol,\gamma}(s,a)\mathrm{d}\nu(s,a). 
\end{align*}
This concludes $\delta(d_{\epol,\gamma},s',a')=0\,\forall (s',a')$. 
\end{proof}
	
\begin{lemma}\label{lem:ratio-estimation}
$\Lw(w_{\epol/\bpol},f)=0$ $\forall f \in L^2(\mathcal{X},\nu):=\{f: \int f(s,a)^2\mathrm{d}\mathrm{\nu}<\infty \}$. Moreover, if we further assume that (a) $d_{\bpol}(s,a) > 0~\forall (s,a)$, (b) $g(s,a)=d_{\epol,\gamma}(s,a)$ if and only if $\delta(g,s',a')=0\,\forall (s',a')$, then $w_{\epol/\bpol}$ is the only function that satisfies such a property.
\end{lemma}

\begin{proof}
Here, we denote $\beta_{\epol/\bpol}(s,a)=\epol(a|s)/\bpol(a|s)$. 
Then,  we have 
	\begin{align*}
	&\Lw(w,f) \\
	&= \rE_{d_{\bpol}}[\gamma w(s,a)f(s',\epol)-w(s,a) f(s,a)]+(1-\gamma)\rE_{\pre \times \epol}[f(s,a)] \\
	&= \rE_{d_{\bpol}}[\gamma w(s,a)\beta_{\epol/\bpol}(a',s')f(s',a')]-\rE_{d_{\bpol}}[ w(s,a)f(s,a)]+(1-\gamma)\rE_{\pre \times \epol}[f(s,a)] \\
    &= \gamma \int f(s',a')P(s'|s,a)\epol(a'|s')d_{\bpol}(s,a)w(s,a)\rD \nu(s,a,s',a')+ \\ 
    &-\int w(s',a')f(s',a')d_{\bpol}(s',a') \rD \nu(s',a')+\int (1-\gamma)\pre(s')\epol(a'|s')f(s',a')\rD \nu(a',s') \\
	&= \int \delta(\tilde g,s',a')f(s',a')\rD\nu(s',a'),
	\end{align*}
	where $\tilde g(s,a)=d_{\bpol}(s,a)w(s,a)$. Note that $\rE_{d_{\bpol}}[\cdot]$ means the expectation with respect to $d_{\bpol}(s,a)P(s'|s,a)\bpol(a'|s')$. 
	
\textbf{First statement }

We prove that $\Lw(w_{\epol/\bpol},f)=0\, \forall f\in L^{2}(\Xcal,\nu)$.  This follows because 
\begin{align*}
    \Lw(w_{\epol/\bpol},f)= \int \delta(d_{\epol,\gamma},s',a')f(s',a')\rD\nu(s',a')=0. 
\end{align*}
Here, we have used \cref{lem:characterization}: $\delta(d_{\epol,\gamma},s',a')=0\,\forall(s',a')$.

\textbf{Second  statement }	

We prove the uniqueness part. Assume $\Lw(w,f)=0\,\forall f\in L^2(\Xcal,\nu)$ holds. Noting the 
\begin{align*}
   \Lw(w,f)=\langle \delta (\tilde g,s',a'),f(s',a') \rangle, 
\end{align*}
where the inner product is for Hilbert space $L^2(\Xcal,\nu)$, the Riesz representative of the functional $f\to \Lw(w,f)$ is $\delta(\tilde g,s',a')$. From the Riesz representation  theorem and the assumption $\Lw(w,f)=0\,\forall f\in L^2(\Xcal,\nu)$, the Riesz representative is uniquely determined as $0$, that is, $\delta(\tilde g,s',a')=0$.

From the assumption (b), this implies $\tilde g=d_{\epol,\gamma}$. From the assumption (a) and the definition of $\tilde g$, this implies $w=w_{\epol/\bpol}$. This concludes the proof. 
\end{proof}

\begin{proof}[\textbf{Proof of Theorem \ref{thm:ipw}}]
We prove two helper lemmas first. 

\begin{lemma}\label{lem:loss_1}
	\begin{align*}
	\Lw(w,f)=\rE_{d_{\bpol}}[\{w_{\epol/\bpol}(s,a) -w(s,a)\}\prod f(s,a)],
	\end{align*}
	where $\prod f(s,a)=f(s,a)-\gamma \rE_{s' \sim P(s,a), a'\sim \epol(s') }[f(s',a' )]$. 
\end{lemma}

\begin{proof}[Proof of Lemma \ref{lem:loss_1}]
	\begin{align*}
	\Lw(w,f) &= \Lw(w,f)- \Lw(w_{\epol/\bpol},f)\\ 
	&= \rE_{d_{\bpol}}[\{\gamma \{w(s,a)-w_{\epol/\bpol}(s,a)\} \beta_{\epol/\bpol}(a',s')f(a',s')] \\
	&-\rE_{d_{\bpol}}[\{w(s,a)-w_{\epol/\bpol}(s,a)\}f(s,a)]\\
	&=\rE_{d_{\bpol}}[\{w_{\epol/\bpol}(s,a) -w(s,a)\}\prod f(s,a)].  \tag*{\qedhere}
	\end{align*}
\end{proof}

\begin{lemma}\label{lem:loss_2}
Define
	\begin{align*}
	f_{g}(s,a)=\rE_{\epol}\left[\sum_{t=0}^{\infty}\gamma^{t}g(s_t,a_t)|s_0=s,a_0=a\right].
	\end{align*}
Here, the expectation is taken with respect to the density $P(s_1|s_0,a_0)\epol(a_1|s_1)P(s_2|s_1,a_1)\cdots$. Then, $f=f_g$ is a solution to $g=\prod f$. 
\end{lemma}

\begin{proof}[Proof of Lemma \ref{lem:loss_2}]
	\begin{align*}
	& \prod f_{g}(s,a) \\
	&=f_{g}(s,a)-\gamma \rE_{s' \sim P(s,a), a'\sim \epol(s')}[f_g(s',a')] \\ 
	&= \rE_{\epol}\left[\sum_{t=0}^{\infty}\gamma^{t}g(s_t,a_t)|s_0=s,a_0=a\right]-\rE_{\epol}\left[\sum_{t=0}^{\infty}\gamma^{t+1}g(s_{t+1},a_{t+1})|s_0=s,a_0=a\right] \\
	&=\rE_{\epol}[g(s_0,a_0)|s_0=s,a_0=a]=g(s,a). \tag*{\qedhere}
	\end{align*}
\end{proof}

We continue with the main proof of Theorem~\ref{thm:ipw}. Here, we have  $\Lw(w,Q^{\epol})=R_{\epol}-\Rw[w]$ since
\begin{align*}
	\Lw(w,Q^{\epol})&=\rE_{d_{\bpol}}[\{w_{\epol/\bpol}(s,a)-w(s,a)\}\prod Q^{\epol}(s,a) ] \\ 
	&=\rE_{d_{\bpol}}[\{w_{\epol/\bpol}(s,a)-w(s,a)\}\rE[r|s,a]  ] \\
	&=\rE_{d_{\bpol}}[\{w_{\epol/\bpol}(s,a)-w(s,a)\}r ]= R_{\epol}-\Rw[w]. 
	\end{align*}
In the first line, we have used \cref{lem:loss_1}. From the first line to the second line, we have used \cref{lem:loss_2}. 

Therefore, if $Q^{\epol}\in \conv(\Fcal)$, for any $w$, 
\begin{align*}
    |R_{\epol}-\Rw[w]|= |\Lw(w,Q^{\epol})|\leq \max_{f\in \conv(\Fcal)}|\Lw(w,f)|=\max_{f\in \Fcal}|\Lw(w,f)|. 
\end{align*}
Here, we have used a fact that $\max_{f\in \conv(\Fcal)}|\Lw(w,f)|=\max_{f\in \Fcal}|\Lw(w,f)|$. This is proved as follows. First, $ \max_{f\in \conv(\Fcal)}|\Lw(w,f)|$ is equal to $|\Lw(w,\tilde f)|$ where $\tilde f=\sum \lambda_i f_i$ and $\sum \lambda_i=1$. Since $|\Lw(w,\tilde f)|\leq \sum \lambda_i |\Lw(w,f_i)|\leq  \max_{f\in \Fcal}|\Lw(w,f)|$, we have 
$$ \max_{f\in \conv(\Fcal)}|\Lw(w,f)|\leq \max_{f\in \Fcal}|\Lw(w,f)|.$$ The reverse direction is obvious. 

Finally, from the definition of $\hatw$, 
\begin{align*}
    |R_{\epol}-\Rw[\hatw]|\leq \min_{w\in \Wcal}\max_{f\in \Fcal}|\Lw(w,f)|.  \tag*{\qedhere}
\end{align*}
\end{proof}

\begin{remark} \label{rem:approx}
When $Q^{\epol} \in \conv(\Fcal)$ is only approximately satisfied, it is straightforward to incorporate the approximation errors into Theorem~\ref{thm:ipw}, where the theorem statement becomes:
\begin{align*}
	\ts
	|R_{\epol}-\Rw[\hat w]| &\leq \max_{f\in \Fcal}|\Lw(w,f)| +\max_{w^{\dagger}\in \mathcal{W}} \min_{q^{\dagger}\in \mathcal{F}}|\Lw(w^{\dagger},Q^{\epol}-q^{\dagger})|.
\end{align*}

For completeness we include the proof below:
\begin{proof}
For any $w$ and for any $q^{\dagger}\in \mathcal{F}$, 
\begin{align*}
	|R_{\epol}-\Rw[w]|= |\Lw(w,Q^{\epol})|
	&\leq |\Lw(w,q^{\dagger})|+ |\Lw(w,Q^{\epol}-q^{\dagger})| \\
	&\leq \max_{f\in \Fcal}|\Lw(w,f)|+  |\Lw(w,Q^{\epol}-q^{\dagger})|. 
\end{align*}
Therefore, for any $w$, 
\begin{align*}
	|R_{\epol}-\Rw[w]| & \leq \max_{f\in \Fcal}|\Lw(w,f)|+  \min_{q^{\dagger}\in \mathcal{F}}|\Lw(w,Q^{\epol}-q^{\dagger})|.
\end{align*}
Then, for any $w$, 
\begin{align*}
	|R_{\epol}-\Rw[w]|
	&\leq \max_{f\in \Fcal}|\Lw(w,f)|+  \max_{w^{\dagger}\in \mathcal{W}} \min_{q^{\dagger}\in \mathcal{F}}|\Lw(w^{\dagger},Q^{\epol}-q^{\dagger})|.
\end{align*}
Thus, 
\begin{align*}
	|R_{\epol}-\Rw[\hat w]|\leq  \min_{w\in \mathcal{W}}\max_{f\in \Fcal}|\Lw(w,f)|+ \max_{w^{\dagger}\in \mathcal{W}} \min_{q^{\dagger}\in \mathcal{F}}|\Lw(w^{\dagger},Q^{\epol}-q^{\dagger})| .
\end{align*}
where $\hat w= \argmin_{w\in \mathcal{W}}\max_{f\in \Fcal}|\Lw(w,f)| $.  
Notice that the additional term $\max_{w^{\dagger}\in \mathcal{W}}\min_{q^{\dagger}\in \mathcal{F}}|\Lw(w^\dagger,Q^{\epol}-q^{\dagger})|$ becomes $0$ when $Q^{\epol} \in \conv(\mathcal{F})$, in which case we recover Theorem~\ref{thm:ipw}. 
\end{proof}
\end{remark}

\begin{proof}[\textbf{Proof of Lemma \ref{lem:kernel}}]
We have 
	\begin{align}
	&\Lw(w,f)^{2}  \nonumber\\
	&=\left \{\rE_{d_{\bpol}}[\gamma w(s,a)  \rE_{a'\sim \epol(s')}[f(s',a')]-w(s,a)f(s,a)]+(1-\gamma)\rE_{\pre \times \epol}[f(s,a)]\right \}^{2} \label{eq:kernel1}\\
	&= \{\rE_{d_{\bpol}}[\gamma w(s,a)  \rE_{a'\sim \epol(s')}[\langle f,K((s',a'),\cdot) \rangle_{\ch_K}]-w(s,a)\langle f,K((s,a),\cdot) \rangle_{\ch_K}] \label{eq:kernel2} \\
	&+(1-\gamma)\rE_{\pre \times \epol}[\langle f,K((s,a),\cdot) \rangle_{\ch_K}]\}^{2} \nonumber \\
	&= \langle f, f^{*}\rangle_{\ch_K}^{2}, \label{eq:kernel3}
	\end{align}
	where $$f^{*}(\cdot)=\rE_{d_{\bpol}}[\gamma w(s,a)\rE_{a'\sim \epol(s')}[K((s',a'),\cdot)] -w(s,a)K((s,a),\cdot) ]+(1-\gamma)\rE_{\pre \times \epol}[K((s,a),\cdot)].$$
	Here, from \eqref{eq:kernel1} to \eqref{eq:kernel2}, we have used a reproducing property of RKHS; $f(s,a)=\langle f(\cdot),K((s,a),\cdot)\rangle$. From \eqref{eq:kernel2} to \eqref{eq:kernel3}, we have used a linear property of the inner product. 
	
	Therefore, 
	\begin{align*}
	\max_{f \in \mathcal{F}} \Lw(w,f)^{2}=\max_{f \in \mathcal{F}} \langle f,f^{*} \rangle^{2}_{\ch_H}=\langle f^{*},f^{*} \rangle^{2}_{\ch_H}. 
	\end{align*}
	from Cauchy–-Schwarz inequality. This is equal to 
\begin{align*}
\max_{f \in \mathcal{F}} \Lw(w,f)^{2}&= \rE_{d_{\bpol}}[\gamma^{2} w(s,a)w(\tilde{s},\tilde{a})\rE_{a'\sim \epol(s'),\tilde a'\sim \epol(\tilde s')}[K((s',a'),(\tilde{s}',\tilde{a}'))]] + \\
&+\rE_{d_{\bpol}}[w(s,a)w(\tilde s,\tilde a)K((s,a),(\tilde s,\tilde a))]\\
&+(1-\gamma)^{2}\rE_{\pre \times \epol}[K((s,a),(\tilde s,\tilde a))] \\
&-2\rE_{d_{\bpol}}[\gamma w(s,a)w(\tilde s,\tilde a)\rE_{a' \sim \epol(s')}[K((s',a'),(\tilde s,\tilde a))] ]  \\
& +2\gamma(1-\gamma) \rE_{(s,a)\sim d_{\bpol},(\tilde s,\tilde a)\sim \pre \times \epol} [\gamma w(s,a)\rE_{a' \sim \epol(s')}[K((s',a'),(\tilde s,\tilde a))] ] \\
&-2(1-\gamma)\rE_{(s,a)\sim d_{\bpol},(\tilde s,\tilde a)\sim \pre \times \epol }[w(s,a)w(\tilde s,\tilde a)K((s,a),(\tilde s,\tilde a)) ],
\end{align*}
where the first expectation is taken with respect to the density 
$d_{\bpol}(s,a,s')d_{\bpol}(\tilde s,\tilde a,\tilde s')$. 

For example, the term $(1-\gamma)^{2}\rE_{\pre \times \epol}[K((s,a),(\tilde s,\tilde a))]$ is derived by 
\begin{align*}
    & \langle (1-\gamma)\rE_{\pre \times \epol}[K((s,a),\cdot)],    (1-\gamma)\rE_{\pre \times \epol}[K((s,a),\cdot)] \rangle_{\ch_K} \\
    &= (1-\gamma)^2\left \langle \int K((s,a),\cdot)\pre(s)\pi(a|s)\nu(a,s)
    ,\int K((\tilde s,\tilde a),\cdot)\pre(\tilde s)\pi(\tilde a| \tilde s)\nu( \tilde a,\tilde s) \right \rangle_{\ch_K}\\
     &= (1-\gamma)^2\int 
      \langle  K((s,a),\cdot),K((\tilde s,\tilde a),\cdot)  \rangle_{\ch_K} \pre(s)\pi(a|s)\pre(\tilde s)\pi(\tilde a| \tilde s)\rD \nu(\tilde a,\tilde s,\tilde a',\tilde s') \\
    &= (1-\gamma)^2\int 
 K((s,a),(\tilde s,\tilde a))  \pre(s)\pi(a|s)\pre(\tilde s)\pi(\tilde a| \tilde s)\rD \nu(\tilde a,\tilde s,\tilde a',\tilde s').  
\end{align*}
Other term are derived in a similar manner. Here, we have used a kernel property 
\begin{align}
\langle K((s,a),\cdot), K((\tilde s,\tilde a),\cdot)\rangle_{\ch_K}= K((s,a),(\tilde s,\tilde a)).\tag*{\qedhere}
\end{align}
\end{proof}

Next we show the result mentioned in the main text, that the minimizer of $\max_{f \in \Fcal} \Lw(w, f)$ is unique when $\Fcal$ corresponds to an ISPD kernel. 
\begin{theorem}\label{thm:key2}
	Assume $\Wcal$ is realizable, i.e., $w_{\epol/\bpol} \in \mathcal{W}$ and conditions (a), (b) in \cref{lem:ratio-estimation}. 
	Then for $\Fcal=L^2(\mathcal{X},\nu)$,  $\hatw(s,a)=w_{\epol/\bpol}(s,a)$ is the unique minimizer of $\max_{f\in\Fcal}\Lw(w, f)$. The same result holds when $\mathcal{F}$ is a RKHS associated with a integrally strictly positive definite (ISPD) kernel $K(\cdot,\cdot)$ \citep{SriperumbudurBharath2010Hsea}.
\end{theorem}


\begin{proof}

	The first statement is obvious from \cref{lem:ratio-estimation} and the proof is omitted, and here we prove the second statement on the ISPD kernel case. If we can prove Lemma \ref{lem:ratio-estimation} when replacing $L^{2}(\mathcal{X},\nu)$ with RKHS associated with an ISPD kernel $K(\cdot,\cdot)$, the statement is concluded. More specifically, what we have to prove is 
	\begin{lemma}\label{lem:mercer}
	Assume conditions in Theorem \ref{thm:key2}. Then, $\Lw(w;f)=0,\,\forall f \in \mathcal{F}$ holds if and only if $w(s,a)=w_{\epol/\bpol}(s,a)$.
	\end{lemma}
	
	This is proved by Mercer's theorem \citep{mohri2012foundations}. From Mercer's theorem, there exist an orthnormal basis $(\phi_j)_{j=1}^{\infty}$ of $L^{2}(\mathcal{X},\nu)$ such that RKHS is represented as 
	\begin{align*}
	\mathcal{F}=\left \{f=\sum_{j=1}^{\infty}b_j \phi_j; (b_j)_{j=1}^{\infty}\in l^{2}(\bbN)\, \mathrm{with}\, \sum_{j=1}^{\infty} \frac{\beta^{2}_j}{\mu_j}<\infty\right \},
	\end{align*}
	where each $\mu_j$ is a positive value since kernel is ISPD. 
	Suppose there exists $w(s,a)\neq w_{\epol/\bpol}(s,a)$ in $w(s,a)\in \mathcal{W}$ satisfying $\Lw(w,f)=0,\,\forall f \in \mathcal{F}$. Then, by taking $b_j=1\,(j=j'),\,b_{j}=1\,(j\neq j')$, for any $j'\in \bbN$, we have $\Lw(w,\phi_{j'})=0$. This implies $\Lw(w,f)=0,\,\forall f \in L^2(\mathcal{X},\nu)=0$. This contradicts the original Lemma \ref{lem:ratio-estimation}. Then, Lemma \ref{lem:mercer} is concluded. 
\end{proof}

\subsection{Proof for Example~\ref{exm:oracle_mwl}} \label{app:oracle_mwl}
Suppose $w_0(s,a) = C ~\forall (s,a)$. Then, for $f =  Q^{\epol}$ we have
\begin{align*}
\Lw(w_0, f) &= C\rE_{d_{\bpol}}[\{\gamma V^{\epol}(s')- Q^{\epol}(s,a)\}]+(1-\gamma)\rE_{\pre \times \epol}[Q^{\epol}(s,a)] \\
&= C R_{\bpol}-R_{\epol}. 
\end{align*}
Hence $C = R_{\epol}/R_{\bpol}$ satisfies that $\Lw(w_0, f) = 0$, $\forall f\in\Fcal$. From Theorem~\ref{thm:ipw}, $|R_{\epol} - \Rw[\hatw]| \le \min_{w}\max_{f}|\Lw(w,f)| \le \max_{f}|\Lw(w_0,f)| = 0$. 

\subsection{The RKHS Result} \label{app:mwl_rkhs}
\begin{lemma}[Formal version of Lemma~\ref{lem:mwl_rkhs}]\label{lem:kernel}
	Let $\langle \cdot,\cdot \rangle _{\ch_K}$ be the inner product of $\ch_K$, which satisfies the reproducing property $f(x)=\langle f,K(\cdot,x)\rangle_{\ch_K}$. 
	When $\mathcal{F}=\{f \in (\Xcal \to \RR): \langle f,f\rangle_{\ch_K} \leq 1\}$, the term $\max_{f\in \mathcal{F}}\Lw(w,f)^{2}$ has a closed form expression:
	\begin{align*}
	&\max_{f \in \mathcal{F}} \Lw(w,f)^{2} \\
	&= \rE_{d_{\bpol}}[\gamma^{2} w(s,a)w(\tilde{s},\tilde{a})\rE_{a'\sim \epol(s'),\tilde a'\sim \epol(\tilde s')}[K((s',a'),(\tilde{s}',\tilde{a}'))]] + \\
	&+\rE_{d_{\bpol}}[w(s,a)w(\tilde s,\tilde a)K((s,a),(\tilde s,\tilde a))]+(1-\gamma)^{2}\rE_{\pre \times \epol}[K((s,a),(\tilde s,\tilde a))]+ \\
	&-2\rE_{d_{\bpol}}[\gamma w(s,a)w(\tilde s,\tilde a)\rE_{a' \sim \epol(s')}[K((s',a'),(\tilde s,\tilde a))] ]  \\
	& +2\gamma(1-\gamma) \rE_{(s,a)\sim d_{\bpol},(\tilde s,\tilde a)\sim \pre \times \epol} [\gamma w(s,a)\rE_{a' \sim \epol(s')}[K((s',a'),(\tilde s,\tilde a))] ] \\
	&-2(1-\gamma)\rE_{(s,a)\sim d_{\bpol},(\tilde s,\tilde a)\sim \pre \times \epol }[w(s,a)w(\tilde s,\tilde a)K((s,a),(\tilde s,\tilde a)) ].
	\end{align*}
	In the above expression, $(s,a) \sim \pre \times \epol$ means $s \sim \pre$, $ a\sim \epol(s)$. All $a'$ and $\tilde{a}'$ terms are marginalized out in the inner expectations, and when they appear together they are always independent. Similarly, in the first 3 lines when $(s,a,s')$ and $(s,a,\tilde{s}')$ appear together in the outer expectation, they are i.i.d.~following the distribution specified in the subscript. 
\end{lemma}

\subsection{Details of Example~\ref{exm:linear_mwl} (LSTDQ)} \label{app:mwl_linear}
Here we provide the detailed derivation of Eq.\eqref{eq:linear}, that is, the closed-form solution of MWL when both $\Wcal$ and $\Fcal$ are set to the same linear class. We assume that the matrix being inverted in Eq.\eqref{eq:linear} is non-singular. 

We show that Eq.\eqref{eq:linear} is a solution to the objective function of MWL by showing that $\Lwn(w(\cdot; \hat{\alpha}), f) = 0$ for any $f$ in the linear class. This suffices because for any $w$, the loss $\Lwn(w(\cdot; \hat{\alpha}), f)^2$ is non-negative and at least $0$, so any $w$ that achieves $0$ for all $f\in\Fcal$ is an argmin of the loss. 

Consider $w\in\Wcal$ whose parameter is $\alpha$ and $f\in\Fcal$ whose parameter is $\beta$. Then
\begin{align*}
\Lwn(w,f)&=\rE_n[\gamma \alpha^{\top}\phi(s,a)\phi(s',\epol)^{\top} \beta  - \alpha^{\top}\phi(s,a)\phi(s,a)^{\top}\beta ]+(1-\gamma)\rE_{\pre}[\phi(s,\epol)^{\top} \beta ] \\
&= \left(\alpha^{\top} \rE_n[\gamma \phi(s,a)\phi(s',\epol)^{\top}  - \phi(s,a)\phi(s,a)^{\top}]+(1-\gamma)\rE_{\pre \times \epol}[\phi(s,a)^{\top}] \right) \beta.
\end{align*}
Since $\Lwn(w,f)$ is linear in $\beta$, to achieve $\max_f \Lwn(w,f)^2 = 0$ it suffices to satisfy
\begin{align} \label{eq:linear2}
\alpha^{\top} \rE_n[\gamma \phi(s,a)\phi(s',\epol)^{\top} - \phi(s,a)\phi(s,a)^{\top}]+(1-\gamma)\rE_{\pre}[\phi(s,\epol)^{\top}] =\bzero. 
\end{align}
Note that this is just a set of linear equations where the number of unknowns is the same as the number of equations, and the $\hat{\alpha}$ in Eq.\eqref{eq:linear} is precisely the solution to Eq.\eqref{eq:linear2} when the matrix multiplied by $\alpha^\top$ is non-singular. 

\subsection{Connection to Dual DICE \citep{ChowYinlam2019DBEo}} \label{app:dualdice}

\citet{ChowYinlam2019DBEo} proposes an extension of \citet{liu2018breaking} without the knowledge of the behavior policy, which shares the same goal with our MWL in Section~\ref{sec:mwl}. 
In fact, there is an interesting connection between our work and theirs, as our key lemma \eqref{lem:ratio-estimation} can be obtained if we take the functional gradient of their loss function. (Ideal) DualDICE with the chi-squared divergence $f(x)=0.5x^{2}$ is described as follows;
\begin{itemize}
	\item Estimate $\nu(s,a)$;
	\begin{align}\label{eq:dual}
	\min_{\nu}0.5 \rE_{d_{\bpol}}[\{\nu(s,a)-(\mathcal{B}\nu)(s,a)\}^{2}]-(1-\gamma)\rE_{\pre \times \epol}[\nu(s,a)],
	\end{align}
	where $(\mathcal{B}\nu)(s,a)=\gamma \rE_{s'\sim P(s,a), a' \sim \epol(s')}[\nu(s',a')]$. 
	\item Estimate the ratio as $\nu(s,a)-(\mathcal{B}\nu)(s,a)$.
\end{itemize}

Because this objective function includes an integral in $(\mathcal{B}\nu)(s,a)$, the Monte-Carlo approximation is required. However, even if we take an Monte-Carlo sample for the approximation, it is biased. Therefore, they further modify this objective function into a more complex minimax form. See (11) in \cite{ChowYinlam2019DBEo}. 

Here, we take a functional derivative of \eqref{eq:dual}(Gateaux derivative) with respect to $\nu$. The functional derivative at $\nu(s,a)$ is 
\begin{align*}
f(s,a)\to \rE_{d_{\bpol}}[\{\nu(s,a)-(\mathcal{B}\nu)(s,a)\}\{f(s,a)-(\mathcal{B} f)(s,a) \}]
-(1-\gamma)\rE_{\pre \times \epol}[f(s,a)].
\end{align*}
The first order condition exactly corresponds to our Lemma~\ref{lem:ratio-estimation}:
\begin{align*}
-\Lw(w,f) &=\rE_{d_{\bpol}}[w(s,a) \{ f(s,a)-\gamma \rE_{s'\sim P(s,a), a' \sim \epol(s')}[f(s',a')]\}]+(1-\gamma)\rE_{\pre \times \epol}[f(s,a)]=0\\
& \iff \rE_{d_{\bpol}}[w(s,a) \{ f(s,a)-\gamma \rE_{ a' \sim \epol(s')}[f(s',a')]\}]+(1-\gamma)\rE_{\pre \times \epol}[f(s,a)]=0. 
\end{align*}
where $w(s,a)=\nu(s,a)-(\mathcal{B}\nu)(s,a)$. Our proposed method with RKHS enables us to directly estimate $w(s,a)$ in one step, and in contrast their approach requires two additional steps: estimating $(\mathcal{B}\nu)(s,a)$ in the loss function, estimating $\nu(s,a)$ by minimizing the loss function, and taking the difference  $ \nu(s,a)- (\mathcal{B}\nu)(s,a)$.

\subsection{Connection between MSWL \citep{liu2018breaking} and Off-policy LSTD}\label{sec:mswl}
Just as our methods become LSTDQ using linear function classes (Examples~\ref{exm:linear_mwl} and \ref{exm:linear_mql}), here we show that the method of \citet{liu2018breaking} (which we call MSWL for easy reference) corresponds to off-policy LSTD. Under our notations, their method is\footnote{Here we have simplified their loss for the discounted case (i.e., their Eq.(15)) with an identity $\rE_{d_{\bpol}}[w(s) f(s)] = \gamma \rE_{d_{\bpol}}[w(s') f(s')] + (1-\gamma) \rE_{d_0}[w(s) f(s)]$ that holds when $d_{\bpol}$ is the discounted occupancy  of the behavior policy (which \citet{liu2018breaking} assume but we do not). Replacing $\rE_{d_{\bpol}}[w(s) f(s)]$ in Eq.\eqref{eq:mswl1} with the RHS of this identity will recover \citet{liu2018breaking}'s original loss. 
Also, because of this reason, their original loss only applies when $d_{\bpol}$ is the discounted occupancy of the behavior policy, whereas this simplified loss (as well as the losses of MWL and MQL) applies even if $d_{\bpol}$ is not a valid occupancy but an arbitrary exploratory distribution. \label{ft:dpib}}
\begin{align} \label{eq:mswl1}
\argmin_{w\in\Wcal^\Scal}\max_{f\in \Fcal^\Scal}\left\{\rE_{d_{\bpol}}\left[ \left (\gamma w(s)f(s')\frac{\epol(a|s)}{\bpol(a|s)}-  w(s)f(s)\right)\right]+(1-\gamma)\rE_{\pre}[f(s)]\right\}^{2}. 
\end{align}
We call this method MSWL (minimax state weight learning) for easy reference. A slightly different but closely related estimator is
\begin{align}\label{eq:mswl2}
      \argmin_{w\in\Wcal^\Scal}\max_{f\in \Fcal^\Scal}\left\{\rE_{d_{\bpol}}\left[\frac{\epol(a|s)}{\bpol(a|s)} \left (\gamma w(s)f(s')-  w(s)f(s)\right)\right]+(1-\gamma)\rE_{\pre}[f(s)]\right\}^{2}. 
\end{align}
Although the two objectives are equal in expectation, under empirical approximations the two estimators are different. In fact, Eq.\eqref{eq:mswl2} corresponds to the most common form of off-policy LSTD  \citep{BertsekasDimitriP2009Pemf} when both $\Wcal^\Scal$ and $\Fcal^\Scal$ are linear (similar to Example~\ref{exm:linear_mwl}). In the same linear setting, Eq.\eqref{eq:mswl1} corresponds to another type of off-policy LSTD discussed by \citet{dann2014policy}. In the tabular setting, we show that  cannot achieve the semiparametric lower bound in Appendix \ref{sec:efficiency}. 

\section{Proofs and Additional Results of Section~\ref{sec:mql} (MQL)} \label{app:mql}

\begin{proof}[\textbf{Proof for \cref{lem:mql}}]

\textbf{First statement}

We prove $\Lq(Q^{\epol},g)=0\,\forall g\in L^2(\Xcal,\nu)$.
The function $Q^{\epol}$ satisfies the following Bellman equation; 
\begin{align}
\rE_{r\sim \Rcal(s,a),\,s'\sim P(s,a)}[r+\gamma Q^{\epol}(s',\epol)-Q^{\epol}(s,a)]=0. 
\end{align}
Then, $\forall g\in L^2(\Xcal,\nu)$, 
\begin{align*}
 0 &=\int \left \{ \rE_{r\sim \Rcal(s,a),\,s'\sim P(s,a)}[r+\gamma Q^{\epol}(s',\epol)-Q^{\epol}(s,a)]\right \}g(s,a)d_{\bpol}(s,a)\rd\nu(s,a) \\
&= \Lq(Q^{\epol},g). 
\end{align*}

\textbf{Second statement}

We prove the uniqueness part. Recall that $Q^{\pi_e}$ is uniquely characterized as  \citep{BertsekasDimitriP2012Dpao}: $\forall (s,a)$,
    \begin{align}\label{eq:bellq}
        \rE_{r\sim \Rcal(s,a),\,s'\sim P(s,a)}[r+\gamma q(s',\epol)-q(s,a)]=0. 
    \end{align}

Assume 
\begin{align}\label{eq:assm}
        \rE_{(s,a,r,s')\sim d_{\bpol}}[\{r+\gamma q(s',\pi_e)-q(s,a)\}g(s,a)]=0,\forall g(s,a)\in L^{2}(\Xcal,\nu)
\end{align}
Note that the left hand side term is seen as 
$$\Lq(q,g)=\langle \{\rE_{r \sim \Rcal(s,a)}[r]+\rE_{s' \sim P(s,a)}[\gamma q(s',\pi_e)]-q(s,a)\}d_{\pi}(s,b), g(s,a) \rangle$$
where the inner product $\langle \cdot,\cdot \rangle$ is for Hilbert space $L^{2}(\Xcal,\nu)$ and the representator of the functional $g \to \Lq(q,g)
$ is $\{\rE[r|s,a]+\gamma q(s',\pi_e)-q(s,a)\}d_{\bpol}(s,a)$. 
From Riesz representation theorem and the assumption \eqref{eq:assm}, the representator of the linear bounded functional $g \to \Lq(q,g)$ is uniquely determined as $0$. Since we also assume $d_{\bpol}(s,a)>0\,\forall (s,a)$, we have 
$\forall (s,a)$,
    \begin{align}
        \rE_{r\sim \Rcal(s,a),\,s'\sim P(s,a)}[r+\gamma q(s',\epol)-q(s,a)]=0. 
    \end{align}
From \eqref{eq:bellq}, 
such $q$ is $Q^{\epol}$. 
\end{proof}

\begin{proof}[\textbf{Proof of Theorem \ref{thm:minimaxq}}]

We prove the first statement. For fixed any $q$, we have 
	\begin{align*}
	|R_{\epol}-\Rq[q]|&= |(1-\gamma)\rE_{\pre \times \pi_e}[q(s,a)]-R_{\epol}| \\
	&= |\rE_{d_{\bpol}}[-\gamma w_{\epol/\bpol}(s,a)q(s', \epol)+w_{\epol/\bpol}(s,a)q(s,a)]-\rE_{d_{\bpol}}[w_{\epol/\bpol}(s,a)r] |\\ 
	&=|\rE_{d_{\bpol}}[w_{\epol/\bpol}(s,a)\{r+\gamma q(s', \epol)-q(s,a)\}]|\\
	&\leq \max_{g\in \conv(\mathcal{G})}|\Lq(q,g)|=\max_{g\in \mathcal{G}}|\Lq(q,g)|. 
	\end{align*}
From the first line to the second line, we use Lemma \ref{lem:ratio-estimation} choosing $f(s,a)$ as $q(s,a)$. From second line to the third line, we use $w_{\epol/\bpol} \in \conv(\Gcal)$. 
	
Then, the second statement follows immediately based on the definition of $\hatq$. 
\end{proof}

\begin{remark}
Similar to Remark~\ref{rem:approx} for Theorem~\ref{thm:ipw}, it is also straightforward to incorporate the approximation errors of $w_{\epol/\bpol} \in \conv(\Gcal)$ into Theorem~\ref{thm:minimaxq}, and the more general bound is
\begin{align*}
	\ts 
	|R_{\epol}-\Rq[\hat q]|& \leq  \min_{q\in \mathcal{Q}}\max_{g\in \Gcal}|\Lq(q,g)|+ \max_{q^{\dagger}\in \mathcal{Q}} \min_{w^{\dagger}\in \mathcal{G}}|\Lq(q^{\dagger},w_{\epol/\bpol}-w^{\dagger})| .
\end{align*}
The proof is omitted due to similarity to the MWL case.
\end{remark}

\begin{proof}[\textbf{Proof of Lemma \ref{lem:kernel2}}]
We have 
	\begin{align*}
	\max_{g\in \mathcal{G}}\Lq(q,g)^{2} &= \max_{g\in \mathcal{G}}\rE_{d_{\bpol}}[g(s,a)(r+\gamma q(s', \epol)-q(s,a))]^{2} \\
	&= \max_{g\in \mathcal{G}}\rE_{d_{\bpol}}[\langle g,K((s,a),\cdot) \rangle_{\ch_K}(r+\gamma q(s', \epol)-q(s,a))]^{2} \\
	&= \max_{g\in \mathcal{G}} \langle g,\rE_{d_{\bpol}}[K((s,a),\cdot) (r+\gamma q(s', \epol)-q(s,a))] \rangle_{\ch_K}^{2} \\
	&= \max_{g\in \mathcal{G}} \langle g,g^{*} \rangle_{\ch_K}^{2} =\langle g^{*},g^{*}\rangle_{\ch_K}
	\end{align*}
where $$g^{*}(\cdot)= \rE_{d_{\bpol}}[K((s,a),\cdot) (r+\gamma q(s', \epol)-q(s,a))].$$
	From the first line to the second line, we use a reproducing property of RKHS; $g(s,a)=\langle g(\cdot),K((s,a),\cdot\rangle_{\ch_K}$. From the second line to the third line, we use a linear property of the inner product. From third line to the fourth line, we use a Cauchy--schwarz inequality since $\mathcal{G}=\{g;\langle g,g\rangle_{\ch_K}\leq 1\}$. 

Then, the last expression $\langle g^{*},g^{*}\rangle_{\ch_K}$ is equal to 
	\begin{align*}
	&\langle \rE_{d_{\bpol}}[K((s,a),\cdot) (r+\gamma q(s', \epol)-q(s,a))],\rE_{d_{\bpol}}[K((s,a),\cdot) (r+\gamma q(s', \epol)-q(s,a))]\rangle_{\ch_K} \\
	&=\langle \rE_{d_{\bpol}}[K((s,a),\cdot) (\rE_{r\sim \Rcal(s,a)}[r]+\rE_{s'\sim P(s,a)}[\gamma q(s', \epol)]-q(s,a))], \\
	&\rE_{d_{\bpol}}[K((\tilde s,\tilde a),\cdot)(\rE_{\tilde r\sim \Rcal(\tilde s,\tilde a)}[\tilde r]+\rE_{\tilde s'\sim P(\tilde s,\tilde a)}[\gamma v(\tilde s')]-q(\tilde s,\tilde a))] \rangle   \\ 
	&=\rE_{d_{\bpol}}[\Delta^{q}(q;s,a,r,s')\Delta^{q}(q;\tilde s,\tilde a,\tilde r,\tilde s')K((s,a),(\tilde s,\tilde a))]    
	\end{align*}
	where $\Delta^{q}(q;s,a,r,s')=r+ \gamma q(s', \epol)-q(s,a)$ and the expectation is taken with respect to the density $d_{\bpol}(s,a,r,s')d_{\bpol}(\tilde s,\tilde a,\tilde r,\tilde s').$ Here, we have used a kernel property 
\begin{align*}
\langle K((s,a),\cdot), K((\tilde s,\tilde a),\cdot)\rangle_{\ch_K}= K((s,a),(\tilde s,\tilde a)). \tag*{\qedhere}
\end{align*}
\end{proof}

\begin{theorem}
\label{thm:key3}
	Assume $Q^{\epol}$ is included in $\mathcal{Q}$ and $d_{\bpol}(s,a)>0\,\forall (s,a)$. Then, if $\mathcal{G}$ is $L^{2}(\mathcal{X},\nu)$, $\hatq=Q^{\epol}$. Also if $\mathcal{G}$ is a RKHS associated with an ISPD kernel, $\hatq=Q^{\epol}$
\end{theorem}
\begin{proof}[Proof of Theorem \ref{thm:key3}]
	The first statement is obvious from \cref{lem:mql}. The second statement is proved similarly as Theorem \ref{thm:key2}. 
\end{proof}

\subsection{Proof for Example~\ref{exm:oracle_mql}} \label{app:oracle_mql}
Suppose $q_0(s,a) = C$. Then, for $g = w_{\epol/\bpol}$, we have
\begin{align*}
\Lq(q,g) = &~ \rE_{d_{\bpol}}[w_{\epol/\bpol}(s,a)(r+\gamma C-C)] \\
= &~ \rE_{d_{\bpol}}[w_{\epol/\bpol}(s,a)r] - (1-\gamma) C =  R_{\epol} - (1-\gamma) C.
\end{align*}
Therefore $C = R_{\epol}/(1-\gamma)$ satisfies $\Lq(q, g) = 0$~$\forall g\in\Gcal$. From Theorem~\ref{thm:minimaxq} we further have $\Rq[\hatq] = R_{\epol}$.

\subsection{Minimax Value Learning (MVL)} \label{app:mvl}
We can extend the loss function in Eq.\eqref{eq:onpolicy_mvl} to the off-policy setting as follows: for state-value function $v: \Scal\to\RR$, define the following two losses:
\begin{align}
& \max_{g \in \mathcal{G}^\Scal}\rE_{\bpol}\left[\frac{\epol(a|s)}{\bpol(a|s)}\{r + \gamma v(s')-v(s)\}g(s)\right]^{2}, \label{eq:mvl1}\\
& \max_{g \in \mathcal{G}^\Scal}\rE_{\bpol}\left[\left(\frac{\epol(a|s)}{\bpol(a|s)}\{r + \gamma v(s')\}-v(s)\right)g(s)\right]^{2}. \label{eq:mvl2}
\end{align}
Again, similar to the situation of MSWL as we discussed in Appendix~\ref{sec:mswl}, these two losses are equal in expectation but exhibit different finite sample behaviors. When we use linear classes for both value functions and importance weights, these two estimators become two variants of off-policy LSTD \citep{dann2014policy, BertsekasDimitriP2009Pemf} and coincide with MSWL and its variant.

\section{Proofs and Additional Results of Section~\ref{sec:dr} (DR and Sample Complexity)}

\begin{proof}[\textbf{Proof of Lemma \ref{lem:db}}]
We have 
\begin{align}
	&R[w,q]-R_{\epol} \nonumber\\
	=&R[w,q]-R[w_{\epol/\bpol}(s,a),Q^{\epol}(s,a)] \label{eq:comp1}\\
	=&\rE_{d_{\bpol}}[\{w(s,a)-w_{\epol/\bpol}(s,a)\}\{r-Q^{\epol}(s,a)+\gamma V^{\epol}(s')\}]+ \label{eq:comp2}\\
	&\rE_{d_{\bpol}}[w_{\epol/\bpol}(s,a)\{Q^{\epol}(s,a)-q(s,a)+\gamma q(s', \epol)-\gamma V^{\epol}(s')\}]+(1-\gamma)\rE_{\pre}[q(s, \epol)-V^{\epol}(s)]+ \nonumber\\
	&\rE_{d_{\bpol}}[\{w(s,a)-w_{\epol/\bpol}(s,a)\}\{Q^{\epol}(s,a)-q(s,a)+\gamma q(s', \epol)-\gamma V^{\epol}(s')\}]+ \nonumber\\
	=& \rE_{d_{\bpol}}[\{w(s,a)-w_{\epol/\bpol}(s,a)\}\{Q^{\epol}(s,a)-q(s,a)+\gamma q(s', \epol)-\gamma V^{\epol}(s')\}]. \label{eq:comp3}
	\end{align}
	
From \eqref{eq:comp1} to \eqref{eq:comp2}, this is just by algebra following the definition of $R[\cdot,\cdot]$. From \eqref{eq:comp2} to \eqref{eq:comp3}, 
we use  the following lemma. \qedhere
\begin{lemma}
\begin{align*}
    0  &=\rE_{d_{\bpol}}[w_{\epol/\bpol}(s,a)\{Q^{\epol}(s,a)-q(s,a)+\gamma q(s', \epol)-\gamma V^{\epol}(s')\}]+(1-\gamma)\rE_{\pre}[q(s, \epol)-V^{\epol}(s)],\\
    0  &= \rE_{d_{\bpol}}[\{w(s,a)-w_{\epol/\bpol}(s,a)\}\{r-Q^{\epol}(s,a)+\gamma V^{\epol}(s')\}].
\end{align*}
\end{lemma}
\begin{proof}
The first equation comes form \cref{lem:ratio-estimation} with $f(s,a)=q(s,a)-Q^{\pi}(s,a)$. The second equation comes from \cref{lem:mql} with $g(s,a)=w(s,a)-w_{\epol/\bpol}(s,a)$.  
\end{proof}
\end{proof}

\begin{proof}[\textbf{Proof of Theorem \ref{thm:db2}}]
We begin with the second statement, which is easier to prove from Lemma \ref{lem:db}: 
\begin{align*}
 R[w',q]-R_{\epol} & =\rE_{d_{\bpol}}[\{w'(s,a)-w_{\epol/\bpol}(s,a)\}\{-\gamma V^{\epol}(s')-\gamma q(s',\epol)+\gamma q(s', \epol)+Q^{\epol}(s,a)\}] \\
 &=\Lq(q,w'-w_{\epol/\bpol})-\Lq(Q^{\epol},w'-w_{\epol/\bpol}) \\
  &=- \Lq(q,w_{\epol/\bpol}-w')-0. \tag{Lemma~\ref{lem:mql}}
\end{align*}
Thus, if $(w_{\epol/\bpol}-w')\in \conv(\Gcal)$
\begin{align*}
|R[w',q]-R_{\epol}|\leq \max_{g \in \conv(\Gcal)}|\Lq(q,g)|=\max_{g \in \Gcal}|\Lq(q,g)|.
 \end{align*}

Next, we prove the first statement. 
From Lemma \ref{lem:db}, 
\begin{align*}
R[w,q']-R_{\epol} & =\rE_{d_{\bpol}}[\{w(s,a)-w_{\epol/\bpol}(s,a)\}\{-\gamma V^{\epol}(s')-\gamma q'(s',\epol)+\gamma q(s', \epol)+Q^{\epol}(s,a)\}] \\
& =\Lw(w,q'-Q^{\epol})- \Lw(w_{\epol/\bpol},q'-Q^{\epol}) \\
&=-\Lw(w,Q^{\epol}-q') - 0. \tag{Lemma~\ref{lem:ratio-estimation}}
\end{align*}
Then, if $(Q^{\epol}-q')\in \conv(\Fcal)$, 
$$	|R[w,q']-R_{\epol} |\leq \max_{f\in\conv(\Fcal)}|\Lw(w,f)|=\max_{f\in\Fcal}|\Lw(w,f)|. $$
Finally, from the definition of $\hatw$ and $\hatq$, we also have 
\begin{align*}
    |R[\hatw,q']-R_{\epol} | \leq \min_{w\in \Wcal}\max_{f\in \Fcal}|\Lw(w,f)| ,\,
    |R[w', \hatq]-R_{\epol} | \leq \min_{q \in \Qcal}\max_{g\in\Gcal}|\Lq(q,g)|.  \tag*{\qedhere}
\end{align*}
\end{proof}

\begin{proof}[\textbf{Proof of Theorem \ref{thm:radema}}]
	We prove the first statement. The second statement is proved in the same way. 
	We have 
	\begin{align}
	&|R[ \hatwn,q']-R_{\epol}| \nonumber\\
	&\leq \max_{f \in \mathcal{F}} |\Lw(\hatwn,f)| \nonumber\\
	&=\max_{f \in \mathcal{F}} |\Lw(\hatwn,f)|-\max_{f \in \mathcal{F}} |\Lwn(\hatwn,f)|+\max_{f \in \mathcal{F}} |\Lwn(\hatwn,f)|-\max_{f \in \mathcal{F}} |\Lw(\hatw,f)|+\max_{f \in \mathcal{F}} |\Lw(\hatw,f)|\nonumber\\
	&\leq \max_{f \in \mathcal{F}} |\Lw(\hatwn,f)|-\max_{f \in \mathcal{F}} |\Lwn(\hatwn,f)+\max_{f \in \mathcal{F}} |\Lwn(\hatw,f)-\max_{f \in \mathcal{F}} |\Lw(\hatw,f)|+\max_{f \in \mathcal{F}} |\Lw(\hatw,f)|  \nonumber \\
	&  \leq 2 \max_{f \in \mathcal{F},w \in \mathcal{W}}||\Lwn(w,f)|- |\Lw(w,f)||+ \min_{w\in \mathcal{W}}\max_{f \in \mathcal{F}} |\Lw(w,f)|.  \label{eq:rade3}
	\end{align}
	
	The remaining task is to bound term $\max_{f \in \mathcal{F},w \in \mathcal{W}}||\Lwn(w,f)|- |\Lw(w,f)||$. This is bounded as follows; 
	\begin{align}\label{eq:rade2}
	\max_{f \in \mathcal{F},w \in \mathcal{W}}||\Lwn(w,f)|- |\Lw(w,f)||\lnsim
	\Rfrak'_n(\mathcal{F},\mathcal{W})+C_fC_w\sqrt{\log(1/\delta)/n}.
	\end{align}
	where $\Rfrak'_n(\mathcal{F},\mathcal{W})$ is the Rademacher complexity of the function class 
	$$\{(s,a,s') \mapsto |w(s,a)(\gamma f(s',\epol)-f(s,a))|: w\in \mathcal{W}, f\in \mathcal{F}\}.$$  Here, we just used an uniform law of large number based on the Rademacher complexity noting $|\left(w(s,a)(\gamma f(s',\epol)-f(s,a))\right)|$ is uniformly bounded by $C_f C_w$ up to some constant \citep[Theorem 8]{bartlett2003rademacher}. From the contraction property of the Rademacher complexity \citep[Theorem 12]{bartlett2003rademacher},
	\begin{align*}\label{eq:contration}
	  	\Rfrak'_n(\mathcal{F},\mathcal{W})\leq 2   	  	\Rfrak_n(\mathcal{F},\mathcal{W})
	\end{align*}
	where $\Rfrak_n(\mathcal{F},\mathcal{W})$ is the Rademacher complexity of the function class
	\begin{align}
\{(s,a,s') \mapsto w(s,a)(\gamma f(s',\epol)-f(s,a)): w\in \mathcal{W}, f\in \mathcal{F}\}. 
	\end{align}
	Finally, Combining \eqref{eq:rade3}, \eqref{eq:rade2} and \eqref{eq:contration}, the proof is concluded. 
\end{proof}

\subsection{Relaxing the i.i.d.~data assumption} \label{app:noniid}
Although the sample complexity results in Section~\ref{sec:dr} are established under i.i.d.~data, we show that under standard assumptions we can also handle dependent data and obtain almost the same results. For simplicity, we only include the result for $\hatwn$. 

In particular, we consider the setting mentioned in Section~\ref{sec:prelim}, that our data is a single long trajectory generated by policy $\bpol$:
$$
s_0, a_0, r_0, s_1, a_1, r_1, \ldots, s_{T-1}, a_{T-1}, r_{T-1}, s_{T}.
$$
We assume that the Markov chain induced by $\bpol$ is ergodic, and $s_0$ is sampled from its stationary distribution so that the chain is stationary. In this case, $d_{\bpol}$ corresponds to such a stationary distribution, which is also the marginal distribution of any $s_t$. We convert this trajectory into a set of transition tuples $\{(s_i, a_i, r_i, s_i')\}_{i=0}^{n-1}$ with $n=T$ and $s_i' = s_{i+1}$, and then apply our estimator on this data. Under the standard $\beta$-mixing condition\footnote{Refer to \cite{MeynS.P.SeanP.2009Mcas} regarding the definition.} \citep[see e.g.,][]{antos2008learning}, we can prove a similar sample complexity result:


\begin{corollary} \label{cor:beta}
	Assume $\{s_i,a_i,r_i,s'_i\}_{i=1}^{n}$ follows a stationary $\beta$--mixing distribution with $\beta$--mixing coefficient $\beta(k)$ for $k=0,1,\cdots$. For any $a_1,a_2>0$ with $2a_1a_2=n$ and $\delta>4(a_1-1)\beta(a_2)$, with probability at least $1-\delta$, we have (all other assumptions are the same as in  Theorem~\ref{thm:radema}(1))
	\begin{align*}
	& |R[\hatwn,q]-R_{\epol}| \lnsim \min_{w \in \mathcal{W}}\max_{f \in \mathcal{F}}|\Lw(w,f)|+\mathfrak{\hat{R}}_{a_1}(\mathcal{F},\mathcal{W})+C_f C_w\sqrt{\frac{\log(1/\delta')}{a_1}}
	\end{align*}
	where $\mathfrak{\hat{R}}_{a_1}(\mathcal{F}, \mathcal{W})$ is the empirical Rademacher complexity of the function class 
	$\{(s,a,s') \mapsto \{w(s,a)(\gamma f(s',\epol)-f(s,a)): w\in \mathcal{W}, f\in \mathcal{F}\}$
	based on a selected subsample of size $a_1$ from the original data (see \citet[Section 3.1]{mohri2009} for details), and $\delta'=\delta-4(a_1-1)\beta(a_2)$. 
\end{corollary}

\begin{proof}[Proof Sketch of \cref{cor:beta}]
We can prove in the same way as for \cref{thm:radema}. The only difference is we use Theorem 2 \citep{mohri2009} to bound the term  $\sup_{f \in \mathcal{F},w \in \mathcal{W}}||\Lwn(w,f)|- |\Lw(w,f)||$. 
\end{proof}
\section{Statistical Efficiency in the Tabular Setting} \label{sec:efficiency}

\subsection{Statistical efficiency of MWL and MQL}
As we already have seen in Example~\ref{exm:linear_mql}, when $\Wcal, \Fcal, \Qcal, \Gcal$ are the same linear class, MWL, MQL, and LSTDQ give the same OPE estimator. These methods are also equivalent in the tabular setting---as tabular is a special case of linear representation (with indicator features)---which also coincides with the model-based (or certainty-equivalent) solution. Below we prove that this tabular estimator can achieve the semiparametric lower bound for infinite horizon OPE \citep{NathanUehara2019}.  \footnote{Semiparametric lower bound is the non--parametric extension of Cramer--Rao lower bound \citep{bickel98}. It is the lower bound of asymptotic MSE among regular estimators \citep{VaartA.W.vander1998As}.} 
Though there are many estimators for OPE, many of the existing OPE methods do not satisfy this property. 

Here, we have the following theorem; the proof is deferred to Appendix~\ref{app:tabular_proof}. 

\begin{theorem}[Restatement of Theorem~\ref{thm:optimal_main}]\label{thm:optimal}
	Assume the whole data set $\{(s,a,r,s')\}$ is geometrically Ergodic \footnote{Regarding the definition, refer to \cite{MeynS.P.SeanP.2009Mcas}}. Then, in the tabular setting, $\sqrt{n}(R_{\mathrm{w},n}[\hat w_n]-R_{\epol})$ and $\sqrt{n}(R_{\mathrm{q}}[\hat q_n]-R_{\epol})$ weakly converge to the normal distribution with mean $0$ and variance
	$$
	\rE_{d_{\bpol}}[w^2_{\epol/\bpol}(s,a)(r+\gamma V^{\epol}(s')-Q^{\epol}(s,a))^{2}].
	$$
This variance matches the semiparametric lower bound for OPE given by \citet[Theorem 5]{NathanUehara2019}.
\end{theorem}

Two remarks are in order:
\begin{enumerate}
	\item Theorem \ref{thm:optimal} could be also extended to the continuous sample space case in a  nonparametric manner, i.e., replacing $\phi(s,a)$ with some basis functions for $L^2$--space and assuming that its dimension grows with some rate related to $n$ and the data--generating process has some smoothness condition \citep{newey94}. The proof is not obvious and we leave it to future work.
	\item In the contextual bandit setting, it is widely known that the importance sampling estimator with plug-in weight from the empirical distribution and the model-based approach can achieve the semiparametric lower bound \citep{HahnJinyong1998OtRo,hirano03}. Our findings are consistent with this fact and is novel in the MDP setting to the best of our knowledge. 
\end{enumerate}  

\begin{table}[ht]
	\centering
	\caption{Summary of the connections between several OPE methods and LSTD, and their optimality in the tabular setting.}
	\begin{tabular}{c|c|c|c|c}
		& MWL & MQL & MSWL & MVL \\\hline
		Definition & Sec~\ref{sec:mwl} & Sec~\ref{sec:mql} & Appendix~\ref{sec:mswl} & Appendix~\ref{app:mvl} \\\hline
		Linear case & \multicolumn{2}{c|}{LSTDQ} & \multicolumn{2}{c}{Off-policy LSTD} \\ \hline
		\tabincell{c}{Optimality \\in tabular} & \multicolumn{2}{c|}{Yes} & \multicolumn{2}{c}{No}
	\end{tabular}
\end{table}

\subsection{Statistical inefficiency of MSWL and MVL for OPE}\label{sec:suboptimal}

Here, we compare the statistical efficiency of MWL, MQL with MSWL, MVL in the tabular setting. First, we show that MSWL, MVL positing the linear class is the same as the off-policy LSTD \citep{BertsekasDimitriP2009Pemf,dann2014policy}. Then, we calculate the asymptotic MSE of these estimators in the tabular case and show that this is larger than the ones of MWL and MQL. 

\paragraph{Equivalence of MSWL, MVL with linear models and off-policy LSTD}

By slightly modifying \citet[Theorem 4]{liu2018breaking}, MSWL is introduced based on the following relation; 
\begin{align*}
\rE_{d_{\bpol}}\left[\frac{\epol(a|s)}{\bpol(a|s)} \left (\gamma \frac{d_{\epol,\gamma}(s)}{d_{\bpol}(s)}f(s')- \frac{d_{\epol,\gamma}(s)}{d_{\bpol}(s)}f(s)\right)\right]+(1-\gamma)\rE_{\pre}[f(s)]=0\,\forall f \in L^{2}(\Scal,\nu).
\end{align*}
Then, the estimator for $\frac{d_{\epol}(s)}{d_{\bpol}(s)}$ is given as 
\begin{align}\label{eq:mswl}
\min_{w^\Scal}\max_{f\in \Fcal^\Scal}\left\{\rE_{d_{\bpol}}\left[\frac{\epol(a|s)}{\bpol(a|s)} \left (\gamma w(s)f(s')-  w(s)f(s)\right)\right]+(1-\gamma)\rE_{\pre}[f(s)]\right\}^{2}. 
\end{align}

As in Example \ref{exm:linear_mwl}, in the linear model case, let $z(s)=\phi(s)^{\top}\alpha$ where $\phi(s)\in \RR^{d}$ is some basis function and $\alpha$ is the parameters. Then, the resulting estimator for $d_{\epol,\gamma}(s)/d_{\bpol}(s)$ is 
\begin{align*}
\hat \alpha = (1-\gamma) \rE_n\left[\frac{\epol(a|s)}{\bpol(a|s)}\left\{-\gamma \phi(s')+\phi(s)\right\}\phi^{\top}(s)\right]^{-1}\rE_{\pre}[\phi(s)]
\end{align*}
Then, the final estimator for $R_{\epol}$ is 
\begin{align*}
(D_{v_1})^{\top}{D^{-1}_{v_2}}D_{v_3}, 
\end{align*}
where
\begin{align*}
D_{v_1}  &=  (1-\gamma)\rE_{\pre}[\phi(s)],  \\
D_{v_2}  &= \rE_n\left[\frac{\epol(a|s)}{\bpol(a|s)}\phi(s)\left\{-\gamma \phi^{\top}(s')+\phi^{\top}(s)\right\}\right] ,\\
D_{v_3} &= \rE_n\left[r\frac{\epol(a|s)}{\bpol(a|s)}\phi(s)\right].\\
\end{align*}

In MVL, the estimator for $V^{\epol}(s)$ is constructed based on the relation;
\begin{align*}
\rE_{d_{\bpol}}\left [\frac{\epol(a|s)}{\bpol(a|s)}\left\{r+\gamma V^{\epol}(s')-V^{\epol}(s)\right \}g(s)\right]=0\,\forall g\in L^{2}(\Scal,\nu). 
\end{align*}
Then, the estimator for $V^{\epol}(s)$ is given by
\begin{align*}
\min_{v \in \Vcal}\max_{g \in \mathcal{G}^\Scal}\rE_{d_{\bpol}}\left[\frac{\epol(a|s)}{\bpol(a|s)}\{r + \gamma v(s')-v(s)\}g(s)\right]^{2}. 
\end{align*}

As in Example \ref{exm:linear_mql}, in the linear model case, let $v(s)=\phi(s)^{\top}\beta$ where $\phi(s)\in \RR^{d}$ is some basis function and $\beta$ is the parameters. Then,
the resulting estimator for $V^{\epol}(s)$ is 
\begin{align*}
\hat{\beta}=    \rE_n\left[\frac{\epol(a|s)}{\bpol(a|s)}\phi(s)\left\{-\gamma \phi^{\top}(s')+\phi^{\top}(s)\right\}\right]^{-1}\rE_n\left[r\frac{\epol(a|s)}{\bpol(a|s)}\phi(s)\right]. 
\end{align*}
Then, the final estimator for $R_{\epol}$ is still $ (D_{v_1})^{\top}D^{-1}_{v_2}D_{v_3}$. This is exactly the same as the estimator obtained by off-policy LSTD \citep{BertsekasDimitriP2009Pemf}. 

\paragraph{Another formulation of MSWL and MVL}

According to \citet[Theorem 4]{liu2018breaking}, we have 
\begin{align*}
\rE_{d_{\bpol}} \left[\left (\gamma \frac{d_{\epol,\gamma}(s)}{d_{\bpol}(s)}\frac{\epol(a|s)}{\bpol(a|s)}- \frac{d_{\epol,\gamma}(s')}{d_{\bpol}(s')}\right)f(s')\right]+(1-\gamma)\rE_{\pre}[f(s)]=0\,\forall f \in L^{2}(\Scal,\nu).
\end{align*}
They construct an estimator for $\frac{d_{\epol,\gamma}(s)}{d_{\bpol}(s)}$  as; 
\begin{align*}
\min_{w\in\Wcal^\Scal}\max_{f\in \Fcal^\Scal}\left\{\rE_{d_{\bpol}}\left[ \left (\gamma w(s) f(s')\frac{\epol(a|s)}{\bpol(a|s)}-w(s)f(s)\right)\right]+(1-\gamma)\rE_{\pre}[f(s)]\right\}^{2}.
\end{align*}
Note that compared with the previous case \eqref{eq:mswl}, the position of the importance weight $\epol/\bpol$ is different. In the same way, MVL is constructed base on the relation;
\begin{align*}
\rE_{d_{\bpol}}\left [\left\{\frac{\epol(a|s)}{\bpol(a|s)}(r+\gamma V^{\epol}(s'))-V^{\epol}(s) \right\} g(s)\right]=0\,\forall g\in L^{2}(\Scal,\nu).  
\end{align*}
The estimator for $V^{\epol}(s)$ is given by 
\begin{align*}
\min_{v \in \Vcal }\max_{g \in \mathcal{G}^\Scal}\rE_{\bpol}\left[\left(\frac{\epol(a|s)}{\bpol(a|s)}\{r + \gamma v(s')\}-v(s)\right)g(s)\right]^{2}.
\end{align*}

When positing linear models, in both cases, the final estimator for $R_{\epol}$ is
\begin{align*}
D_{v_1}^{\top} \{D_{v_4}\}^{-1}  D_{v_3},
\end{align*}
where 
\begin{align*}
D_{v_4}= \rE_n\left[\phi(s)\left\{-\gamma \frac{\epol(a|s)}{\bpol(a|s)}\phi^{\top}(s')+\phi^{\top}(s)\right\}\right]. 
\end{align*}
This is exactly the same as the another type of off-policy LSTD \citep{dann2014policy}.

\paragraph{Statistical inefficiency of MSWL, MVL and off-policy LSTD}

Next, we calculate the asymptotic variance of $D_{v1}D^{-1}_{v2}D_{v3}$ and $D_{v1}D^{-1}_{v4}D_{v3}$ in the tabular setting. It is shown that these methods cannot achieve the semiparametric lower bound \citep{NathanUehara2019}. These results show that these methods are statistically inefficient. Note that this implication is also brought to the general continuous sample space case since the the asymptotic MSE is generally the same even in the continuous sample space case with some smoothness conditions.

\begin{theorem}\label{thm:optimal2}
	Assume the whole data is geometrically Ergodic. In the tabular setting, $\sqrt{n}(D_{v1}D^{-1}_{v2}D_{v3}-R_{\epol})$ weakly converges to the normal distribution with mean $0$ and variance;
	\begin{align*}
	\rE_{d_{\bpol}}\left[\left\{\frac{d_{\epol,\gamma}(s)}{d_{\bpol}(s)}\right\}^{2}\mathrm{var}_{d_{\bpol}}\left[\frac{\epol(a|s)}{\bpol(a|s)}\{r+V^{\epol}(s')-V^{\epol}(s)\}|s\right] \right].
	\end{align*}
	This is larger than the semiparametric lower bound. 
\end{theorem}

\begin{theorem}\label{thm:optimal3}
	Assume the whole data is geometrically Ergodic. In the tabular setting,
	$\sqrt{n}(D_{v1}D^{-1}_{v4}D_{v3}-R_{\epol})$ weakly converges to the normal distribution with mean $0$ and variance;
	\begin{align*}
	\rE_{d_{\bpol}}\left[\left\{\frac{d_{\epol,\gamma}(s)}{d_{\bpol}(s)}\right\}^{2}\mathrm{var}_{d_{\bpol}}\left[\frac{\epol(a|s)}{\bpol(a|s)}\{r+V^{\epol}(s')\}|s\right] \right].
	\end{align*}
	This is larger than the semiparametric lower bound. 
\end{theorem}

\begin{figure}[!htb]
	\centering
	\begin{subfigure}[b]{0.5\textwidth}
		\includegraphics[width=\linewidth]{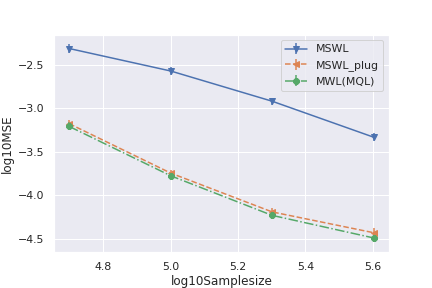}
		\caption{$\alpha=0.2$}\label{fig:2new}
	\end{subfigure}\hfill
	\begin{subfigure}[b]{0.5\textwidth}
		\includegraphics[width=\linewidth]{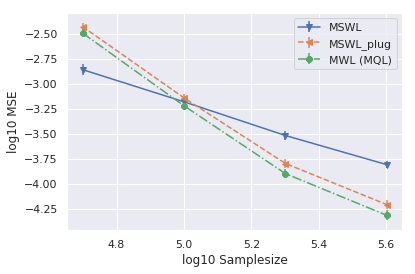}
		\caption{$\alpha=0.4$}\label{fig:4new}
	\end{subfigure}
	\caption{MSE as a function of sample size. $\alpha$ controls the difference between $\bpol$ and $\epol$; see Appendix~\ref{sec:numerical} for details. \label{fig:taxi}}
\end{figure}

\subsection{Experiments}\label{sec:numerical}
We show some empirical results that back up the theoretical discussions in this section. We conduct experiments in the Taxi environment \citep{DietterichT.G.2000HRLw}, which has 20000 states and 6 actions; see \citet[Section  5]{liu2018breaking} for more details. 
We compare three methods, all using the tabular representation: MSWL with exact $\epol$, MSWL with estimated $\epol$ (``plug-in''), and MWL (same as MQL). As we have mentioned earlier, this comparison is essentially among off-policy LSTD, plug-in off-policy LSTD, and LSTDQ. 

We choose the target policy $\epol$ to be the one obtained after running Q-learning for 1000 iterations, and choose another policy $\pi_{+}$ after 150 iterations. The behavior policy is $\bpol=\alpha \epol +(1-\alpha)\pi_{+}$. We report the results for $\alpha \in \{0.2,\,0.4\}$. The discount factor is $\gamma = 0.98$. 


We use a single trajectory and vary the truncation size $T$ as $[5,10,20, 40]\times 10^4$. For each case, by making $200$ replications, we report the Monte Carlo MSE of each estimator with their $95\%$ interval in Figure \ref{fig:taxi}. \nj{Use new legend}

It is observed that MWL is significantly better than MSWL and MWL is slightly better than plug-in MSWL. This is because MWL is statisitcally efficient and MSWL is statistically inefficient as we have shown earlier in this section. The reason why plug-in MSWL is superior to the original MSWL is that the plug-in based on MLE with a well specified model can be viewed as a form of control variates \citep{HenmiMasayuki2004Apcn,HannaJosiahP.2018ISPE}. Whether the plug-in MSWL can achieve the semiparametric lower bound remains as future work. 

\subsection{Proofs of Theorems \ref{thm:optimal}, \ref{thm:optimal2}, and \ref{thm:optimal3}} \label{app:tabular_proof}

\begin{proof}[Proof of Theorem \ref{thm:optimal}]
	Recall that the estimator is written as $D^{\top}_{q1}D^{-1}_{q2}D_{q3}$,
	where
	\begin{align*}
	D_{q1} &= (1-\gamma)\rE_{\pre\times \epol}[\phi(s,a)] \\
	D_{q2} &= \rE_{n}[-\gamma \phi(s,a)\phi^{\top}(s',\epol)+\phi(s,a)\phi(s,a)^{\top}] \\
	D_{q3} &= \rE_{n}[r\phi(s,a)]. 
	\end{align*}
	Recall that $D^{-1}_{q2}D_{q3}=\hat{\beta}$ is seen as Z--estimator with a parametric model $q(s,a;\beta)=\beta^{\top}\phi(s,a)$. More specifically, the estimator $\hat{\beta}$ is given as a solution to 
	\begin{align*}
	\rE_n[\{r+\gamma q(s',\epol;\beta)-q(s,a;\beta)\}\phi(s,a)]=0. 
	\end{align*}
	Following the standard theory of Z--estimator \citep{VaartA.W.vander1998As}, the asymptotic MSE of $\beta$ is calculated as a sandwich estimator; 
	\begin{align*}
	\sqrt{n}(\hat{\beta}-\beta_0) &\stackrel{d}{\rightarrow} \mathcal{N}(0,D^{-1}_1D_2{D^{-1}_1}^{\top})
	\end{align*}
	where $\beta^{\top}_0\phi(s,a)=Q^{\epol}(s,a)$ and 
	\begin{align}
	D_1 &= \rE_{d_{\bpol}}[\phi(s,a)\{-\gamma \phi(s',\epol)+\phi(s,a)\}^{\top}]|_{\beta_0} ,\nonumber \\
	D_2 &= \mathrm{var}_{d_{\bpol}}[\{r+\gamma q(s',\epol;\beta)-q(s,a;\beta)\}\phi(s,a)  ]|_{\beta_0}  \label{eq:optim_proof_1}\\
	&= \rE_{d_{\bpol}}[\mathrm{var}_{d_{\bpol}}[r+\gamma V^{\epol}(s')-Q^{\epol}(s,a)|s,a]\phi(s,a)\phi^{\top}(s,a) ]\label{eq:optim_proof_2} \\
	&+  \mathrm{var}_{d_{\bpol}}[\rE_{d_{\bpol}}[r+\gamma V^{\epol}(s')-Q^{\epol}(s,a)|s,a]\phi(s,a)\phi^{\top}(s,a) ]  \nonumber \\
	&= \rE_{d_{\bpol}}[\mathrm{var}_{d_{\bpol}}[r+\gamma V^{\epol}(s')-Q^{\epol}(s,a)|s,a]\phi(s,a)\phi^{\top}(s,a) ]\label{eq:optim_proof_3}. 
	\end{align}
	Here, we use a variance decomposition to simplify $D_2$ from \eqref{eq:optim_proof_1} to \eqref{eq:optim_proof_2}.  We use a relation $\rE_{d_{\bpol}}[r+\gamma V^{\epol}(s')-Q^{\epol}(s,a)|s,a]=0$ from \eqref{eq:optim_proof_2} to \eqref{eq:optim_proof_3}. Then, by delta method,
	\begin{align*}
	\sqrt{n}(D_{q1}^{\top}D^{-1}_{q2}D_{q3}-R_{\epol})\stackrel{d}{\rightarrow} \mathcal{N}(0,D_{q1}^{\top} D^{-1}_1D_2{D^{-1}_1}^{\top}D_{q1}). 
	\end{align*}
	
	From now on, we simplify the expression $D^{\top}_{q1} D^{-1}_1D_2{D^{-1}_1}^{\top}D_{q1}$. 
	First, we observe 
	\begin{align*}
	[D_{q1}]_{|S|i_1+i_2} = (1-\gamma)\pre(S_{i_1})\epol(A_{i_2}|S_{i_1}), 
	\end{align*}
	where $[D_{q1}]_{|S|i_1+i_2}$ is a element corresponding $(S_{i_1},A_{i_2})$ of $D_{q1}$. In addition, 
	\begin{align*}
	D^{-1}_{1} &=  \rE_{d_{\bpol}}[\phi(s,a)\{-\gamma \phi(s',\epol)+\phi(s,a)\}^{\top}]^{-1} \\
	&=\rE_{d_{\bpol}}[\phi(s,a)\phi^{\top}(s,a)(-\gamma P^{\epol}+I)^{\top} ]^{-1}\\
	&=\{(-\gamma P^{\epol}+I)^{-1}\}^{\top}\rE_{d_{\bpol}}[\phi(s,a)\phi^{\top}(s,a) ]^{-1}.
	\end{align*}
	where $P^{\epol}$ is a transition matrix between $(s,a)$ and $(s',a')$, and $I$ is an identity matrix.   
	
	Therefore, by defining $g(s,a)=\mathrm{var}_{d_{\bpol}}[r+\gamma V^{\epol}(s')-Q^{\epol}(s,a)|s,a]$  and $(I-\gamma P^{\epol})^{-1}D_{q1}=D_3$
	the asymptotic variance is 
	\begin{align*}
	& D^{\top}_3\rE_{d_{\bpol}}[\phi(s,a)\phi^{\top}(s,a) ]^{-1}
	\rE_{d_{\bpol}}[g(s,a) \phi(s,a)\phi(s,a)^{\top}]
	\rE_{d_{\bpol}}[\phi(s,a)\phi^{\top}(s,a) ]^{-1}D_3\\
	&=  \sum_{\tilde s \in \mathcal{S},\tilde a\in \mathcal{A}} \{d_{\bpol}(\tilde s,\tilde a)\}^{-1}g(\tilde s,\tilde a)\{D^{\top}_3 I_{\tilde s,\tilde a}\}^2
	\end{align*}
	where $I_{\tilde s,\tilde a}$ is a $|S||A|$--dimensional vector, which the element corresponding $(\tilde s,\tilde a)$ is 1 and other elements are $0$. Noting $D^{\top}_3 I_{\tilde s,\tilde a}=d_{\epol,\gamma}(\tilde s,\tilde a)$, the asymptotic variance is 
	\begin{align*}
	&\sum_{\tilde s \in \mathcal{S},\tilde a\in \mathcal{A}} \{d_{\bpol}(\tilde s,\tilde a)\}^{-1}g(\tilde s,\tilde a)d^2_{\epol,\gamma}(\tilde s,\tilde a)\\
	&=\rE_{d_{\bpol}}[w^2_{\epol/\bpol}(s,a)\mathrm{var}_{d_{\bpol}}[r+\gamma V^{\epol}(s')-Q^{\epol}(s,a)|s,a]]\\
	&=\rE_{d_{\bpol}}[w^2_{\epol/\bpol}(s,a)(r+\gamma V^{\epol}(s')-Q^{\epol}(s,a))^{2}]. 
	\end{align*}
	This concludes the proof. 
\end{proof}

\begin{proof}[Proof of Theorem \ref{thm:optimal2}]
	
	Recall that $D^{-1}_{v2}D_{v3}=\hat{\beta}$ is seen as Z--estimator with a parametric model $v(s;\beta)=\beta^{\top}\phi(s)$. More specifically, the estimator $\hat{\beta}$ is given as a solution to
	\begin{align*}
	\rE_{n}\left[\frac{\epol(a|s)}{\bpol(a|s)}\{r+v(s';\beta)-v(s;\beta)\}\phi(s) \right]=0. 
	\end{align*}
	
	Following the standard theory of Z--estimator \citep{VaartA.W.vander1998As}, the asymptotic variance of $\beta$ is calculated as a sandwich estimator;
	\begin{align*}
	\sqrt{n}(\hat \beta-\beta_0)\stackrel{d}{\rightarrow}\mathcal{N}(0,D^{-1}_1D_2 (D^{-1}_1)^{\top}),
	\end{align*}
	where $\beta^{\top}_0\phi(s)=V^{\epol}(s)$ and 
	\begin{align*}
	D_1 &= \rE_{d_{\bpol}}[\phi(s)\{-\gamma \phi(s')+\phi(s)\}^{\top}]\\ 
	D_2 &= \mathrm{var}_{d_{\bpol}} \left[\frac{\epol(a|s)}{\bpol(a|s)}\{r+v(s';\beta)-v(s;\beta)\}\phi(s) \right]|_{\beta_0}\\
	&=\rE_{d_{\bpol}}\left[\mathrm{var}_{d_{\bpol}}\left[\frac{\epol(a|s)}{\bpol(a|s)}\{r+\gamma V^{\epol}(s';\beta)-V^{\epol}(s)\}|s\right]\phi(s)\phi^{\top}(s) \right].
	\end{align*}
	Then, by delta method, 
	\begin{align*}
	\sqrt{n}(D^{\top}_{v1}D^{-1}_{v2}D_{v3}-R_{\epol})\stackrel{d}{\rightarrow}\mathcal{N}(0, D^{\top}_{v1}D^{-1}_1D_2 (D^{-1}_1)^{\top}  D_{v1}).
	\end{align*}
	
	From now on, we simplify the expression $ D^{\top}_{v1}S^{-1}_1S_2 (S^{-1}_1)^{\top}  D_{v1}$. First, we observe 
	\begin{align*}
	[D_{v1}]_i = (1-\gamma)\pre(S_i),
	\end{align*}
	where $[D_{v1}]_i$ is $i$--th element. In addition, 
	\begin{align*}
	D^{-1}_1 &= \rE_{d_{\bpol}}[\phi(s)\{-\gamma \phi(s')+\phi(s) \}^{\top}]^{-1} \\
	&= \rE_{d_{\bpol}}[\phi(s)\phi^{\top}(s)\{-\gamma P^{\epol}+I\}^{\top} ]^{-1} \\
	&= (\{-\gamma P^{\epol}+I\}^{\top})^{-1}\rE_{d_{\bpol}}[\phi(s)\phi^{\top}(s)]^{-1},
	\end{align*}
	where $P^{\epol}$ is a transition matrix from the current state to the next state. 
	
	Therefore, by defining $g(s)=\mathrm{var}_{d_{\bpol}}\left[\frac{\epol(a|s)}{\bpol(a|s)}\{r+V^{\epol}(s')-V^{\epol}(s)\}|s\right]$ and $\{-\gamma P^{\epol}+I\}^{-1} D_{v1}=D_3$, the asymptotic variance is 
	\begin{align*}
	& D^{\top}_3 \rE_{d_{\bpol}(s)}[ \phi(s)\phi^{\top}(s) ]^{-1} \rE_{d_{\bpol}(s)}[g(s) \phi(s)\phi^{\top}(s) ] \rE_{d_{\bpol}(s)}[ \phi(s)\phi^{\top}(s) ] ^{-1} D_3 \\
	&=\sum_{\tilde s\in \Scal}d^{-1}_{\bpol(\tilde s)}g(\tilde s)\{D^{\top}_3 I_{\tilde s}\}^{2},
	\end{align*}
	where $I_{\tilde s}$ is $|S|$--dimensional vector, which the element corresponding $\tilde s$ is 1 and other elements are $0$. Noting $D^{\top}_3 I_{\tilde s}=d_{\epol,\gamma}(\tilde s)$, the asymptotic variance is 
	\begin{align*}
	\sum_{\tilde s\in \Scal}d^{-1}_{\bpol(\tilde s)}g(\tilde s)d^2_{\epol}(\tilde s)=
	\rE_{d_{\bpol}}\left[ \left\{\frac{d_{\epol,\gamma}(s)}{d_{\bpol}(s)} \right\}^{2} \mathrm{var}_{d_{\bpol}}\left[\frac{\epol(a|s)}{\bpol(a|s)}\{r+\gamma V^{\epol}(s')-V^{\epol}(s)\}|s\right]\right]. 
	\end{align*}
	
	Finally, we show this is larger than the semiparametric lower bound. This is seen as 
	\begin{align*}
	& \rE_{d_{\bpol}}\left[ \left\{\frac{d_{\epol,\gamma}(s)}{d_{\bpol}(s)} \right\}^{2} \mathrm{var}_{d_{\bpol}}\left[\frac{\epol(a|s)}{\bpol(a|s)}\{r+\gamma V^{\epol}(s')-V^{\epol}(s)\}|s\right]\right] \\
	& \geq  \rE_{d_{\bpol}}\left[ \left\{\frac{d_{\epol,\gamma}(s)}{d_{\bpol}(s)} \right\}^{2} \rE_{d_{\bpol}}\left[ \mathrm{var}_{d_{\bpol}}\left[\frac{\epol(a|s)}{\bpol(a|s)}\{r+\gamma V^{\epol}(s';\beta)-V^{\epol}(s)\}|s,a\right]\right]\right]\\
	& = \rE_{d_{\bpol}}\left[w_{\epol/\bpol}^{2}(s,a)\mathrm{var}[r+\gamma V^{\epol}(s')-Q^{\epol}(s,a)|s,a ]\right]. 
	\end{align*}
	Here, from the first line to the second line, we use a general inequality $\rvar[x]=\rvar[\rE[x|y]]+\rE[\rvar[x|y]]\geq \rE[\rvar[x|y]]$. 
\end{proof}

\begin{proof}[Proof of Theorem \ref{thm:optimal3}]
	
	By refining $g(s)=\rvar\left[\frac{\epol(a|s)}{\bpol(a|s)}\{r+V^{\epol}(s)\}|s\right]$ in the proof of Theorem \ref{thm:optimal2}, 
	the asymptotic variance is 
	\begin{align*}
	\sum_{\tilde s\in \Scal}d^{-1}_{\bpol(\tilde s)}g(\tilde s)d^2_{\epol,\gamma}(\tilde s)=
	\rE_{d_{\bpol}}\left[ \left\{\frac{d_{\epol,\gamma}(s)}{d_{\bpol}(s)} \right\}^{2} \mathrm{var}_{d_{\bpol}}\left[\frac{\epol(a|s)}{\bpol(a|s)}\{r+\gamma V^{\epol}(s')\}|s\right]\right]. 
	\end{align*}
	
	Then, we show this is larger than the semiparamatric lower bound. This is seen as 
	\begin{align*}
	\sum_{\tilde s\in \Scal}d^{-1}_{\bpol(\tilde s)}g(\tilde s)d^2_{\epol,\gamma}(\tilde s) &=
	\rE_{d_{\bpol}}\left[ \left\{\frac{d_{\epol,\gamma}(s)}{d_{\bpol}(s)} \right\}^{2} \mathrm{var}_{d_{\bpol}}\left[\frac{\epol(a|s)}{\bpol(a|s)}\{r+\gamma V^{\epol}(s')\}|s\right]\right] \\ 
	&\geq  \rE_{d_{\bpol}}\left[ \left\{\frac{d_{\epol,\gamma}(s)}{d_{\bpol}(s)} \right\}^{2}\rE_{\bpol}\left[ \mathrm{var}_{d_{\bpol}}\left[\frac{\epol(a|s)}{\bpol(a|s)}\{r+\gamma V^{\epol}(s')\}|s,a\right]\right]\right] \\
	& = \rE_{d_{\bpol}}\left[w_{\epol/\bpol}^{2}(s,a)\mathrm{var}[r+\gamma V^{\epol}(s')-Q^{\epol}(s,a)|s,a ]\right].  \tag*{\qedhere}
	\end{align*}
	
\end{proof}

\section{Experiments in the Function Approximation Setting}
\label{app:exp_fa}
In this section, we empirically evaluate our new algorithms MWL and MQL, and make comparison with MSWL \citep{liu2018breaking} and DualDICE \citep{nachum2019dualdice} in the function approximation setting.

\subsection{Setup}
We consider infinite-horizon discounted setting with $\gamma=0.999$, and test all the algorithms on CartPole, a control task with continuous state space and discrete action space. Based on the implementation of OpenAI Gym \citep{brockman2016openai}, we define a new state-action-dependent reward function and add small Gaussian noise with zero mean on the transition dynamics.

To obtain the behavior and the target policies, we first use the open source code\footnote{https://github.com/openai/baselines} of DQN to get a near-optimal $Q$ function, and then apply softmax on the $Q$ value divided by an adjustable temperature $\tau$:
\begin{equation}
    \pi(a|s)\propto \exp(\frac{Q(s,a)}{\tau})
\end{equation}
We choose $\tau=1.0$ as the behavior policy and $\tau=0.25, 0.5, 1.5, 2.0$ as the target policies. The training datasets are generated by collecting trajectories of the behavior policy with fixed horizon length 1000. If the agent visits the terminal states within 1000 steps, the trajectory will be completed by repeating the last state and continuing to sample actions.

\subsection{Algorithms Implementation}
We use the loss functions and the OPE estimators derived in this paper and previous literature \citep{liu2018breaking, nachum2019dualdice}, except for MWL, in which we find another equivalent loss function can work better in practice:

\begin{equation}
    L = \E_{d_{\pi_0}}[\gamma \Big(w(s,a) f(s',\pi)-w(s',a') f(s',a')\Big)] - (1-\gamma)\E_{d_0\times \pi_0}[w(s,a)f(s,a)] + (1-\gamma)\E_{d_0\times \pi}[f(s,a)]\label{mwl:second_loss}
\end{equation}

In all algorithms, we keep the structure of the neural networks the same, which have two hidden layers with 32 units in each and ReLU as activation function. Besides, the observation is normalized to zero mean and unit variance, and the batch size is fixed to 500. In MSWL, the normalized states are the only input, while in the others, the states and the actions are concatenated and fed together into neural networks.

As for DualDICE, we conduct evaluation based on the open source implementation\footnote{https://github.com/google-research/google-research/tree/master/dual\_dice}. The learning rate of $\nu-$network and $\zeta-$network are changed to be $\eta_\nu=0.0005$ and $\eta_\zeta=0.005$, respectively, after a grid search in $\eta_\nu\times \eta_\zeta \in \{0.0001, 0.0005, 0.001, 0.0015, 0.002\}\times\{0.001, 0.005, 0.01, 0.015, 0.02\}$.

In MSWL, we use estimated policy distribution instead of the true value to compute the policy ratio. To do this, we train a 64x64 MLP with cross-entropy loss until convergence to approximate the distribution of the behavior policy. The learning rate is set to be 0.0005.

Besides, we implement MQL, MWL and MSWL with $\mathcal{F}$ corresponding to a RKHS associated with kernel $K(\cdot, \cdot)$. The new loss function of MWL \eqref{mwl:second_loss} can be written as:
\begingroup
\allowdisplaybreaks
\begin{align*}
    L =& \gamma^2 \E_{s,a,s',\tilde{s}, \tilde{a},\tilde{s}'\sim d_{\pi_0}}[w(s,a)w(\tilde{s}, \tilde{a})\E_{a', \tilde{a}'\sim \pi}[K((s',a'), (\tilde{s}', \tilde{a}'))]] \\
    &+\gamma^2 \E_{s',a',\tilde{s}', \tilde{a}'\sim d_{\pi_0}}[w(s',a')w(\tilde{s}', \tilde{a}')K((s',a'), (\tilde{s}', \tilde{a}'))]\\
    &+(1-\gamma)^2 \E_{s,a,\tilde{s}, \tilde{a}\sim d_0\times \pi_0}[w(s,a)w(\tilde{s}, \tilde{a})K((s,a), (\tilde{s}, \tilde{a}))]\\
    &+(1-\gamma)^2 \E_{s,a,\tilde{s}, \tilde{a}\sim d_0\times \pi}[K((s,a), (\tilde{s}, \tilde{a}))]\\
    &-2\gamma^2 \E_{s,a,s',\tilde{s}', \tilde{a}'\sim d_{\pi_0}}[w(s,a)w(\tilde{s}', \tilde{a}')\E_{a'\sim \pi}[K((s',a'), (\tilde{s}', \tilde{a}'))]]\\
    &-2\gamma(1-\gamma) \E_{s,a,s'\sim d_{\pi_0}, \tilde{s}, \tilde{a}\sim d_0\times \pi_0}[w(s,a)w(\tilde{s}, \tilde{a})\E_{a'\sim \pi}[K((s',a'),(\tilde{s}, \tilde{a}))]]\\
    &+2\gamma(1-\gamma) \E_{s,a,s'\sim d_{\pi_0}, \tilde{s}, \tilde{a}\sim d_0\times \pi}[w(s,a)\E_{a'\sim \pi}[K((s',a'),(\tilde{s}, \tilde{a}))]]\\
    &+2\gamma(1-\gamma) \E_{s',a'\sim d_{\pi_0}, \tilde{s}, \tilde{a}\sim d_0\times \pi_0}[w(s',a')w(\tilde{s}, \tilde{a})K((s',a'),(\tilde{s}, \tilde{a}))]\\
    &-2\gamma(1-\gamma) \E_{s',a'\sim d_{\pi_0}, \tilde{s}, \tilde{a}\sim d_0\times \pi}[w(s',a')K((s',a'),(\tilde{s}, \tilde{a}))]\\
    &-2(1-\gamma)^2 \E_{s,a\sim d_0\times\pi_0, \tilde{s}, \tilde{a} \sim d_0\times \pi}[w(s,a) K((s,a),(\tilde{s}, \tilde{a}))]. \addtocounter{equation}{1}\tag{\theequation}
\end{align*}
\endgroup
We choose  the RBF kernel, defined as
\begin{equation}
    K(\vect{x}_i, \vect{x}_j)=\exp(-\frac{\|\vect{x}_i-\vect{x}_j\|^2_2}{2\sigma^2})
\end{equation}
where $\vect{x}_i,\vect{x}_j$ corresponds to state vectors in MSWL, and corresponds to the vectors concatenated by state and action in MWL and MQL. Denote $h$ as the median of the pairwise distance between $\vect{x}_i$, we set $\sigma$ equal to $h,\frac{h}{3}$ and $\frac{h}{15}$ in MSWL, MWL and MQL, respectively. The learning rates are fixed to 0.005 in these three methods.

In MSWL and MWL, to ensure that the predicted density ratio is non-negative, we apply $\log(1+\exp(\cdot))$ as the activation function in the last layer of the neural networks. Moreover, we normalize the ratio to have unit mean value in each batch, which works better.

\subsection{Results}
We generate $N$ datasets with different random seeds. For each dataset, we run all these four algorithms from the beginning until convergence, and consider it as one trial. Every 100 training iterations, the estimation of value function is logged, and the average over the last five logged estimations will be recorded as the result in this trial. For each algorithm, we report the normalized MSE, defined by
\begin{equation} \label{eq:rel_mse}
    \frac{1}{N}\sum_{i=1}^N \frac{(\hat{R}_{\epol}^{(i)}-R_{\epol})^2}{(R_{\bpol} - R_{\epol})^2}
\end{equation}
where $\hat{R}_{\epol}^{(i)}$ is the estimated return in the $i$-th trial; $R_{\epol}$ and $R_{\bpol}$ are the true expected returns of the target and the behavior policies, respectively, estimated by 500 on-policy Monte-Carlo trajectories (truncated at $H=10000$ steps to make sure $\gamma^H$ is sufficiently small). This normalization is very informative, as a na\"ive baseline that treats $R_{\epol} \approx R_{\bpol}$ will get $0.0$ (after taking logarithm), so any method that beats this simple baseline should get a negative score.  We conduct $N=25$ trials and plot the results in Figure \ref{fig:Exp_Function_Approx}.

\begin{figure*}[h!]
	\centering
	\begin{subfigure}[t]{0.5\textwidth}
        \centering
        \includegraphics[scale=0.3]{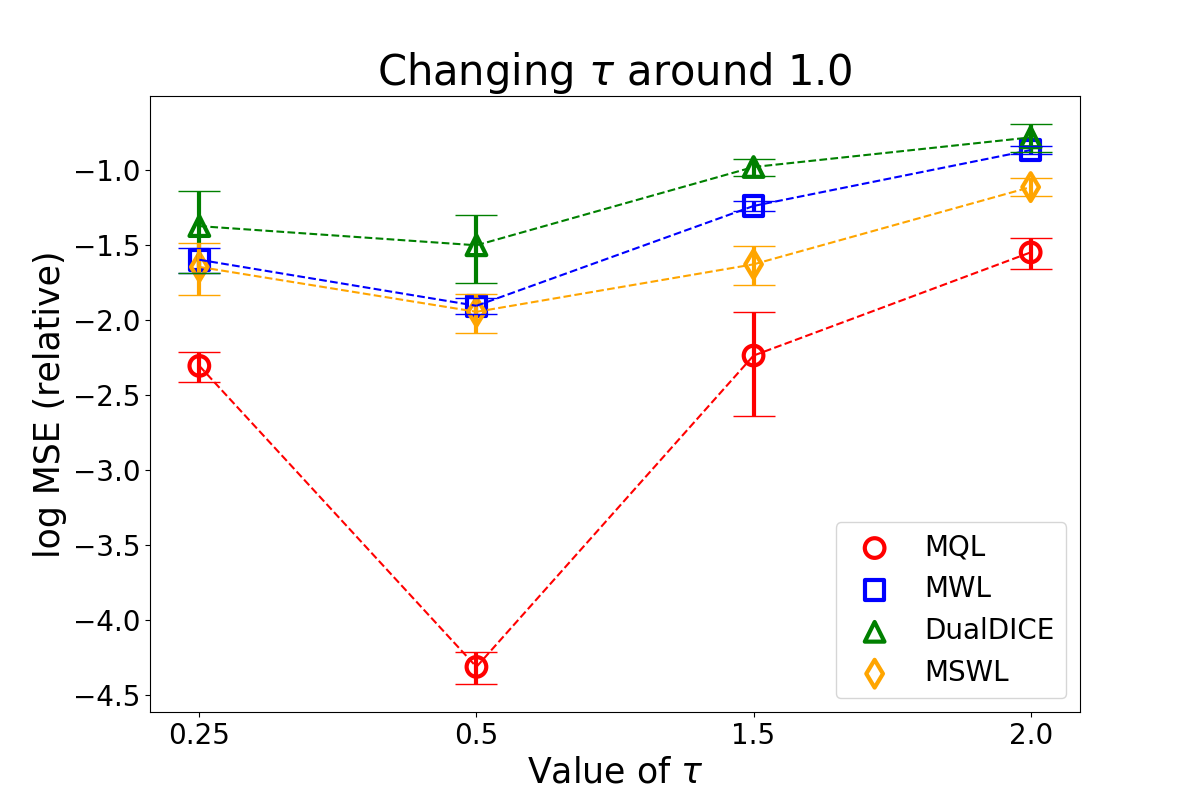}
        \caption*{}
    \end{subfigure}%
    ~ 
    \begin{subfigure}[t]{0.5\textwidth}
        \centering
        \includegraphics[scale=0.3]{Exp_Vary_DataSize.png}
        \caption*{}
    \end{subfigure}
	\caption{A Comparison of four algorithms: MQL (red circle), MWL (blue square), DualDICE (green triangle) and MSWL (yellow diamond). \textbf{Left}: We fix the number of trajectories to 200 and change the target policies. \textbf{Right}: We keep the target policy as $\tau=1.5$, and vary the number of samples.\label{fig:Exp_Function_Approx}}
\end{figure*}

\subsection{Error bars}
We denote $x_i=\frac{(\hat{R}_{\epol}^{(i)}-R_{\epol})^2}{(R_{\bpol} - R_{\epol})^2}$. As we can see, the result we plot (Eq.\eqref{eq:rel_mse}) is just the average of $N$ i.i.d.~random variables $\{x_i\}_{i=1}^n$. We plot twice the standard error of the estimation---which corresponds to $95\%$ confidence intervals---under logarithmic transformation. That is, the upper bound of the error bar is
$$
\log(\frac{1}{N}\sum_{i=1}^N x_i + \frac{2\sigma}{\sqrt{N}})
$$
and the lower bound of the error bar is $$
\log(\frac{1}{N}\sum_{i=1}^N x_i - \frac{2\sigma}{\sqrt{N}})
$$
where $\sigma$ is the sample standard deviation of $x_i$.

\section{Step-wise IS as a Special Case of MWL} \label{app:stepIS}
We show that step-wise IS \citep{precup2000eligibility} in discounted episodic problems can be viewed as a special case of MWL, and sketch the proof as follows. In addition to the setup in Section~\ref{sec:prelim}, we also assume that the MDP always goes to the absorbing state in $H$ steps from any starting state drawn from $\pre$. 
The data are trajectories generated by $\bpol$. We first convert the MDP into an equivalent \emph{history-based} MDP, i.e., a new MDP where the state is the history of the original MDP (absorbing states are still treated specially). We use $h_t$ to denote a history of length $t$, i.e., $h_t = (s_0, a_0, r_0, s_1, \ldots, s_t)$. Since the history-based MDP still fits our framework, we can apply MWL as-is to the history-based MDP. In this case, each data trajectory will be converted into $H$ tuples in the form of $(h_t, a_t, r_t, h_{t+1})$. 

We choose the following $\Fcal$ class for MWL, which is the space of \emph{all} functions over histories (of various lengths up to $H$). Assuming all histories have non-zero density under $\epol$ (this assumption can be removed), from Lemma~\ref{lem:ratio-estimation-informal} we know that the only $w$ that satisfies $\forall f\in\Fcal, ~\Lw(w,f)=0$ is
$$
\frac{d_{\epol, \gamma}(h_t, a_t)}{d_{\bpol}(h_t, a_t)} = \frac{(1-\gamma)\gamma^t}{1/H} \prod_{t'=0}^t \frac{\epol(a_{t'}|s_{t'})}{\bpol(a_{t'}|s_{t'})}. ~\footnote{Here the term $\frac{(1-\gamma)\gamma^t}{1/H}$ appears because a state at time step $t$ is discounted in the evaluation objective but its empirical frequency in the data is not. Other than that, the proof of this equation is precisely how one derives sequential IS, i.e., density ratio between histories is equal to the cumulative product of importance weights on actions.}
$$
Note that $\Rw[w]$ with such an $w$ is precisely the step-wise IS estimator in discounted episodic problems. Furthermore, the true marginalized importance weight in the original MDP $\frac{d_{\epol, \gamma}(s,a)}{d_{\bpol}(s,a)}$ is not feasible under this ``overly rich'' history-dependent discriminator class (see also related discussions in \citet{jiang2019value}).

\end{document}